\documentclass[10pt]{article}
\usepackage[a4paper, total={6in, 8in}]{geometry}
\usepackage[usenames,dvipsnames,svgnames,table]{xcolor} 

\usepackage{algorithm}
\usepackage{algorithmic}
\floatname{algorithm}{Algorithm}

\usepackage{longtable}
\usepackage{hhline}
\usepackage{float}
\newfloat{algorithm}{tb}{lop}

\usepackage{enumitem} 
\usepackage{rotfloat} 
\usepackage{graphicx}
\usepackage{epsfig}
\usepackage{amssymb, amsmath, amsthm}
\usepackage{tikz}
\usetikzlibrary{3d, calc}
\usepackage{verbatim}
\usepackage{hyperref}
\usepackage{natbib}
\usepackage{authblk}
\usepackage{multirow}
\usepackage{subcaption} 
\usepackage{array}
\newcolumntype{L}{>{\centering\arraybackslash}m{0.1\linewidth}}

\newtheorem{theorem}{Theorem}[section]
\newtheorem{lemma}[theorem]{Lemma}

\DeclareMathOperator*{\argmin}{argmin}

\newcommand{\Exp}{\textrm{Exp}}
\newcommand{\Log}{\textrm{Log}}

\newcommand{\bfv}{\boldsymbol{v}}
\newcommand{\bfw}{\boldsymbol{w}}
\newcommand{\bfx}{\boldsymbol{x}}
\newcommand{\bfy}{\boldsymbol{y}}
\newcommand{\bfz}{\boldsymbol{z}}
\newcommand{\bfmu}{\boldsymbol{\mu}}

\newcommand{\bfM}{\boldsymbol{M}}

\newcommand{\bfQ}{\boldsymbol{Q}}

\newcommand{\bfI}{\boldsymbol{I}}
\newcommand{\bfU}{\boldsymbol{U}}
\newcommand{\bfW}{\boldsymbol{W}}

\newcommand{\bfSigma}{\boldsymbol{\Sigma}}

\newcommand{\bbN}{\mathbb{N}}
\newcommand{\bbP}{\mathbb{P}}
\newcommand{\bbQ}{\mathbb{Q}}
\newcommand{\bbR}{\mathbb{R}}
\newcommand{\bbS}{\mathbb{S}}

\newcommand{\bbRpos}{\mathbb{R}^+}

\newcommand{\calI}{\mathcal{I}}
\newcommand{\calM}{\mathcal{M}}
\newcommand{\calN}{\mathcal{N}}
\newcommand{\calW}{\mathcal{W}}

\newcommand{\WL}{\mathcal{WL}}

\newcommand{\Frechet}{Fr\'{e}chet }

\begin{document}
	
\title{Learning over von Mises–Fisher Distributions via a Wasserstein-like Geometry}
\author[1,2]{Kisung You}
\author[2]{Dennis Shung}
\author[2]{Mauro Giuffr\`{e}}
\affil[1]{Department of Mathematics, Baruch College}
\affil[2]{Department of Internal Medicine, Yale University School of Medicine}
\date{}

\maketitle
\begin{abstract}
We introduce a novel, geometry-aware distance metric for the family of von Mises–Fisher (vMF) distributions, which are fundamental models for directional data on the unit hypersphere. Although the vMF distribution is widely employed in a variety of probabilistic learning tasks involving spherical data, principled tools for comparing vMF distributions remain limited, primarily due to the intractability of normalization constants and the absence of suitable geometric metrics. Motivated by the theory of optimal transport, we propose a Wasserstein-like distance that decomposes the discrepancy between two vMF distributions into two interpretable components: a geodesic term capturing the angular separation between mean directions, and a variance-like term quantifying differences in concentration parameters. The derivation leverages a Gaussian approximation in the high-concentration regime to yield a tractable, closed-form expression that respects the intrinsic spherical geometry. We show that the proposed distance exhibits desirable theoretical properties and induces a latent geometric structure on the space of non-degenerate vMF distributions. As a primary application, we develop the efficient algorithms for vMF mixture reduction, enabling structure-preserving compression of mixture models in high-dimensional settings. Empirical results on synthetic datasets and real-world high-dimensional embeddings, including biomedical sentence representations and deep visual features, demonstrate the effectiveness of the proposed geometry in distinguishing distributions and supporting interpretable inference. This work expands the statistical toolbox for directional data analysis by introducing a tractable, transport-inspired distance tailored to the geometry of the hypersphere.
\end{abstract}


\section{Introduction}

One of the enduring trajectories in the evolution of statistical learning is the expansion of its scope to encompass increasingly diverse types of data and the complex structural constraints they may satisfy. Classical statistical theory, grounded in standard Euclidean geometry, often falls short when data naturally resides on curved or constrained domains. A paradigmatic example of this shift is directional statistics \citep{mardia_2000_DirectionalStatistics, ley_2017_ModernDirectionalStatistics}, a field concerned with the analysis of data that is directional in nature—such as orientations, rotations, axes, and subspaces. In this regime, observations no longer adhere to the usual rules of vector spaces, necessitating alternative treatments that account for the underlying geometry.

Among the various forms of non-Euclidean data, the analysis of directions has received particular attention, dating back to early 20th-century studies in geology, meteorology, and biology, well before the formalization of directional statistics as a distinct discipline. In this context, a direction refers to an equivalence class of vectors under positive scalar multiplication. Concretely, this yields a canonical representation of directional data as points on the unit hypersphere, a compact Riemannian manifold with constant positive curvature. The shift from flat to spherical geometry introduces both rich challenges and new opportunities: classical notions such as distances, averages, and probability densities must be redefined to respect the manifold structure. In this paper, we focus on the von Mises–Fisher distribution as a probabilistic model for such data and address the lack of principled tools for measuring distances between these distributions in a way that honors  their intrinsic geometry.

The von Mises–Fisher (vMF) distribution is a fundamental probabilistic model for data constrained to lie on the unit hypersphere. Originally introduced by \citet{fisher_1953_DispersionSphere} and later studied under the name Langevin distribution in both physical and statistical contexts \citep{watson_1984_TheoryConcentratedLangevin, watamori_1996_StatisticalInferenceLangevin}, the vMF distribution serves as a natural analog of the isotropic Gaussian distribution for spherical domains. It is parameterized by a unit-norm mean direction and a concentration parameter that governs the tightness of the distribution around this mean. Despite its conceptual simplicity and geometric appeal, the vMF distribution has remained underutilized in the broader context of statistical inference, particularly in tasks requiring comparison between models.

Prior work on vMF distributions has largely focused on modeling and clustering applications, typically relying on likelihood-based criteria or heuristic similarity measures. A few notable exceptions exist. The Kullback–Leibler (KL) divergence \citep{kullback_1951_InformationSufficiency}, for instance, has been applied in variational inference, albeit only in restricted settings due to the intractability of the normalizing constant \citep{gopal_2014_MisesfisherClusteringModels, diethe_2015_NoteKullbackLeiblerDivergence}. More generally, statistical distances that are both computationally tractable and geometrically meaningful remain elusive for vMF distributions, particularly for tasks such as model reduction, barycenter computation, or hierarchical clustering. Recent work by \citet{kitagawa_2022_MisesFisherDistributionsTheir} introduced analytical expressions for certain dissimilarities within the family of $f$-divergences \citep{csiszar_1963_InformationstheoretischeUngleichungUnd}, but these expressions often involve evaluating complex integrals and do not easily lend themselves to learning tasks involving large collections of vMF distributions.

\begin{figure}[ht]
	\centering
	\begin{subfigure}[b]{0.32\textwidth}
		\centering
		\includegraphics[width=\textwidth]{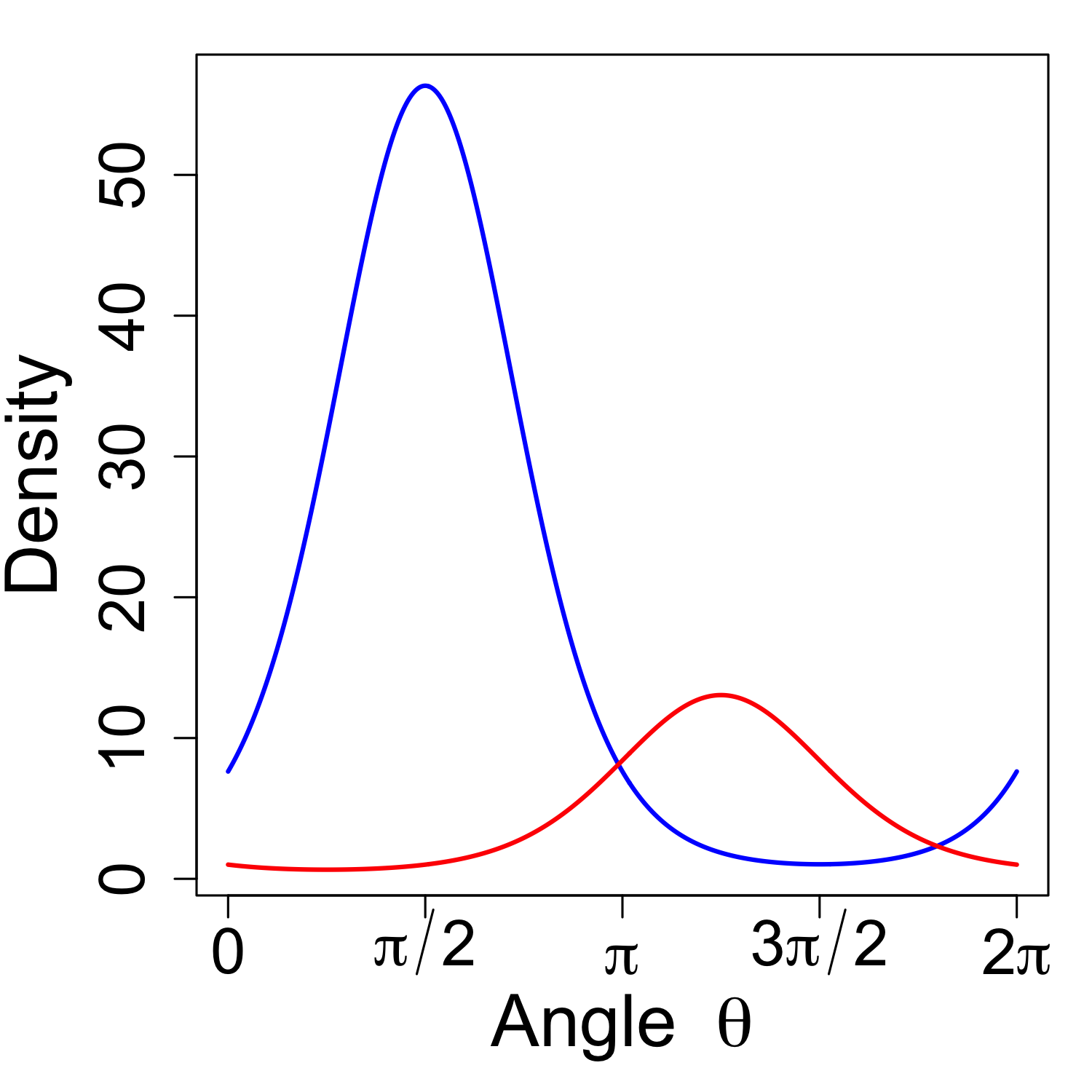}
	\end{subfigure}
	\begin{subfigure}[b]{0.32\textwidth}
		\centering
		\includegraphics[width=\textwidth]{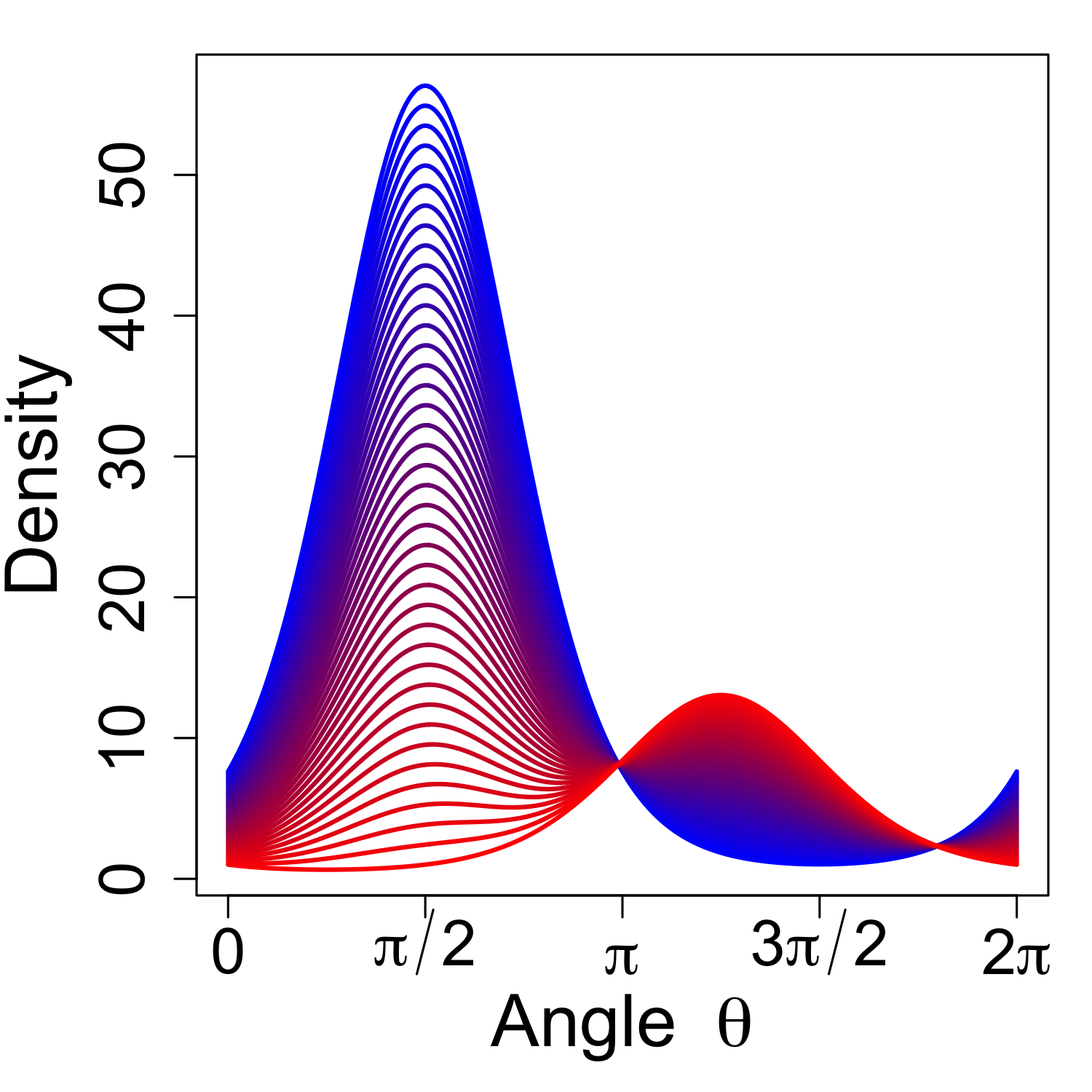}
	\end{subfigure}
	\begin{subfigure}[b]{0.32\textwidth}
		\centering
		\includegraphics[width=\textwidth]{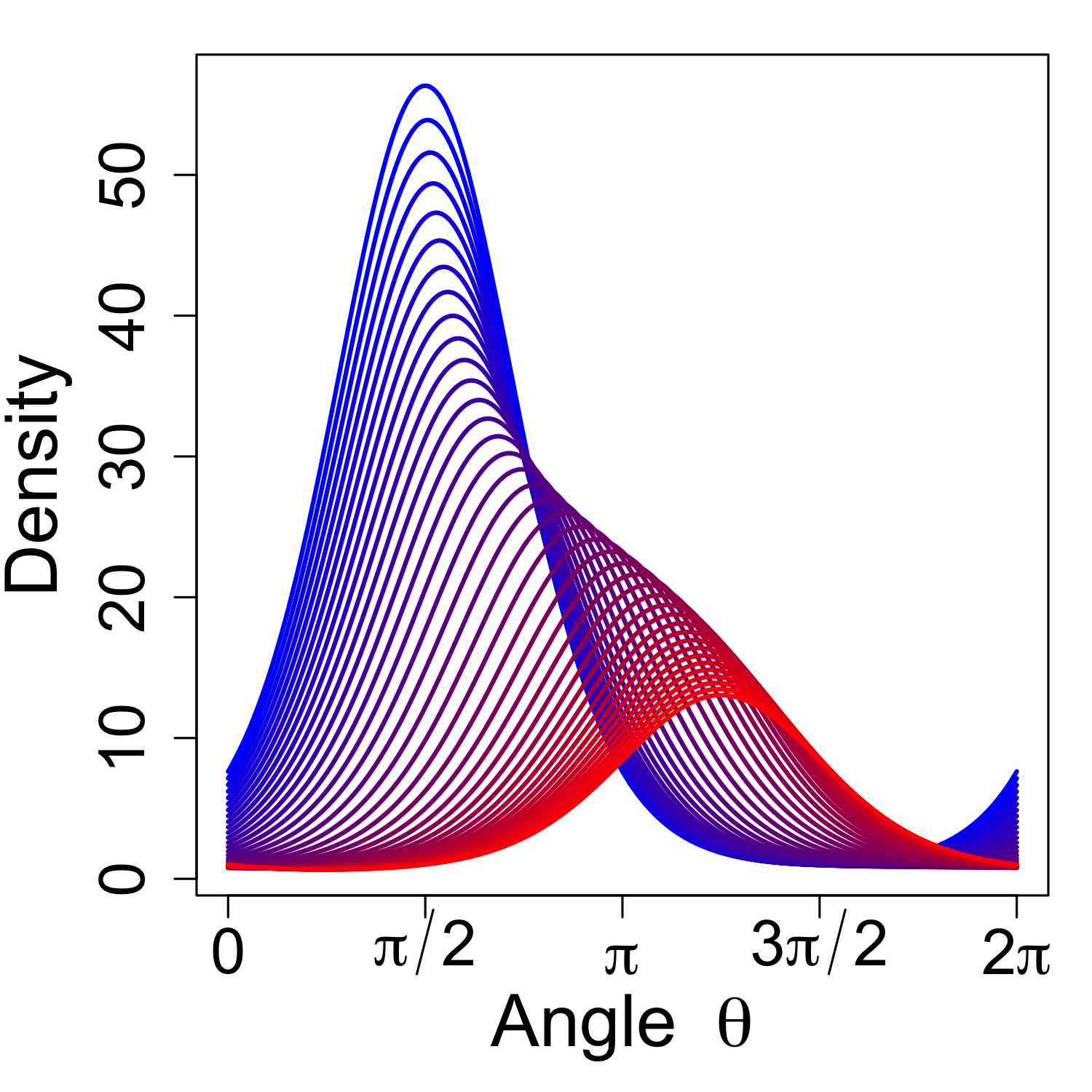}
	\end{subfigure}	
	\caption{	Comparison of interpolations between two distinct von Mises–Fisher distributions (left). Interpolation paths are shown for the standard $L_2$ geometry (middle) and the proposed $\WL$ distance metric (right).	}
	\label{fig:compare-geometry-vMF}
\end{figure}

In this paper, we introduce a new geometry-aware distance between vMF distributions, drawing motivation from the theory of optimal transport. Specifically, we propose a Wasserstein-like distance, denoted $\WL$, that decomposes the discrepancy between two vMF distributions into two interpretable components: a geodesic term that captures minimal displacement in mean direction, and a variance-like term that quantifies differences in concentration. By leveraging the high-concentration regime of vMF distributions, we employ Gaussian approximations on the tangent space of the unit hypersphere to derive a tractable, closed-form expression for $\WL$. The resulting metric satisfies desirable theoretical properties, including continuity, symmetry, and topological consistency, and it reveals a latent geometric structure on the space of non-degenerate vMF distributions. Its transport-based construction also distinguishes it from conventional $L_2$ metrics, as illustrated in Figure~\ref{fig:compare-geometry-vMF}, which contrasts interpolation paths under both geometries, akin to visualizations commonly found in the optimal transport literature.

The primary application we consider is mixture model reduction, where the objective is to compress a mixture of vMF components into a more parsimonious representation while preserving its geometric structure. We show that the $\WL$ distance enables efficient greedy and partitional reduction strategies by facilitating closed-form barycenter computations and performant pairwise dissimilarity evaluation. This leads to practical algorithms for simplifying complex vMF mixtures in high-dimensional settings.

We support our methodology with both synthetic and real-world experiments. Simulations illustrate how the $\WL$ distance meaningfully distinguishes vMF distributions across a range of location and concentration regimes, outperforming standard $L_2$-based methods. Applications to high-dimensional embedding data, including sentence embeddings from biomedical abstracts and deep visual features from image data, demonstrate the utility of the proposed geometry in structure-preserving model compression and interpretable classification. Altogether, this work contributes a new, theoretically grounded distance metric tailored to vMF distributions, expanding the statistical toolbox for inference and representation learning in non-Euclidean settings.

\section{Preliminaries}

We begin by introducing notations and conventions used throughout the paper. Bold lower- and upper-case letters denote vectors and matrices, respectively. Let $\bbS^{d-1} = \lbrace \bfx \in \bbR^d ~|~ \|\bfx\| = 1\rbrace$ denote the unit hypersphere in $\bbR^d$, i.e., the set of unit-norm vectors forming a $(d-1)$-dimensional Riemannian manifold embedded in $\bbR^d$ for $d \geq 1$. We denote by $\bbRpos$ and $\bbR_{\geq 0}$ the sets of strictly positive and nonnegative real numbers. Blackboard bold letters $\bbP_i$, indexed by $i \in \bbN$, are used to denote vMF distributions. The notation $[n]$ for $n \in \bbN$ denotes the set ${1, 2, \ldots, n}$. The orthogonal group in dimension $d$ is denoted by $\mathcal{O}(d)$, consisting of all $d \times d$ orthogonal matrices $\bfQ$ satisfying $\bfQ^\top \bfQ = \bfQ \bfQ^\top = \bfI_d$, where $\bfI_d$ is the $d$-dimensional identity matrix. For a vector $\bfx = [x_1, \ldots, x_n] \in \bbR^n$, we write $\bfx \in \Delta_{n-1}$ to indicate that it lies in the $(n-1)$-dimensional probability simplex, i.e., $\sum_{i=1}^n x_i = 1$ and $x_i > 0$ for all $i \in [n]$.

\subsection{von Mises-Fisher distribution}

Let $\bfx \in \bbR^d$ be a random vector following the von Mises–Fisher distribution $vMF(\bfmu, \kappa)$, where $\bfmu \in \bbS^{d-1}$ is the mean direction and $\kappa \geq 0$ is the concentration parameter. Its probability density function is given by \begin{equation}\label{definition:vMF} p(\bfx \mid \bfmu, \kappa) = C_d (\kappa) \exp \left( \kappa \bfmu^\top \bfx \right), \end{equation} with normalizing constant \begin{equation}\label{definition:normalizingconstant} C_d (\kappa) = \frac{\kappa^{d/2 - 1}}{(2\pi)^{d/2} I_{d/2 - 1}(\kappa)}, \end{equation} where $I_\nu$ denotes the modified Bessel function of the first kind of order $\nu$ \citep{abramowitz_1972_HandbookMathematicalFunctions}. The concentration parameter $\kappa$ inversely reflects the spread of the distribution: higher values of $\kappa$ lead to tighter concentration around the mean direction $\bfmu$, while smaller values yield broader distributions. In the limiting case $\kappa \to 0$, the distribution converges to the uniform distribution over the sphere.

The vMF distribution can be viewed as the spherical analog of the isotropic Gaussian distribution. Consider a Gaussian distribution centered at $\bfmu \in \bbS^{d-1}$ with precision parameter $\kappa$, i.e., inverse variance.  If we constrain the domain to the unit sphere, the corresponding density becomes \begin{equation*} f(\bfx \mid \bfmu, \kappa) \propto \exp \left( -\frac{\kappa}{2}\|\bfx - \bfmu\|^2 \right) = \exp \left( -\frac{\kappa}{2} \bfx^\top \bfx + \kappa \bfmu^\top \bfx - \frac{\kappa}{2} \bfmu^\top \bfmu \right) \propto \exp \left( \kappa \bfmu^\top \bfx \right), \end{equation*} where the last proportionality follows from the fact that both $\|\bfx\| = 1$ and $\|\bfmu\| = 1$, rendering the norm terms constant. This highlights the vMF distribution's rotational symmetry around $\bfmu$. More importantly, it belongs to the canonical exponential family \citep{pitman_1937_SignificanceTestsWhich, koopman_1936_DistributionsAdmittingSufficient}, making it amenable to general exponential-family-based inference techniques for data residing on the hypersphere.

Despite these favorable properties, comparing vMF distributions remains analytically challenging due to the nonlinear dependence of the normalization constant $C_d(\kappa)$ on $\kappa$. For instance, while the KL divergence is a widely used statistical discrepancy, it lacks a closed-form expression for general vMF distributions because the Bessel function $I_\nu$ must be computed via an integral over a bounded domain \citep{abramowitz_1972_HandbookMathematicalFunctions}. The only known exception is when the reference distribution is uniform \citep{diethe_2015_NoteKullbackLeiblerDivergence}, a case which finds practical utility in variational inference \citep{gopal_2014_MisesfisherClusteringModels} and in latent variable models involving spherical priors \citep{xu_2018_SphericalLatentSpaces, davidson_2022_HypersphericalVariationalAutoEncoders}.

\subsection{Wasserstein distance between probability measures}
The Wasserstein distance is a fundamental metric from the theory of optimal transport that quantifies the discrepancy between two probability measures by minimizing the cost of transporting mass from one distribution to another \citep{villani_2003_TopicsOptimalTransportation, villani_2009_OptimalTransportOld}. Let $(\mathcal{X}, d)$ be a complete, separable metric space with distance $d$, and denote by $\mu$ and $\nu$  two probability measures on $\mathcal{X}$ with finite second moments. The squared 2-Wasserstein distance is defined by \begin{equation}\label{eq-wasserstein-definition} \mathcal{W}^2(\mu, \nu) := \inf_{\gamma \in \Gamma(\mu, \nu)} \int_{\mathcal{X} \times \mathcal{X}} d^2(x, y)  d\gamma(x, y), \end{equation} where $\Gamma(\mu, \nu)$ denotes the set of all couplings of $\mu$ and $\nu$, i.e., joint distributions $\gamma$ on $\mathcal{X} \times \mathcal{X}$ with marginals $\mu$ and $\nu$.

The Wasserstein distance arises naturally from Monge's formulation of optimal transport and has gained widespread use in modern statistics and machine learning due to its sensitivity to geometric structure. Notably, it metrizes weak convergence plus convergence of second moments, thus offering robust geometric interpretations and favorable inferential behavior \citep{panaretos_2020_InvitationStatisticsWasserstein, you_2024_WassersteinMedianProbability}.

When $\mathcal{X} = \bbR^d$ and both $\mu$ and $\nu$ are Gaussian with means $\boldsymbol{m}_1$, $\boldsymbol{m}_2$ and covariance matrices $\bfSigma_1$, $\bfSigma_2$, the Wasserstein distance admits a closed-form expression: \begin{equation}\label{eq-wasserstein-gaussian-general} \mathcal{W}^2(\mathcal{N}(\boldsymbol{m}_1, \bfSigma_1), \mathcal{N}(\boldsymbol{m}_2, \bfSigma_2)) = \|\boldsymbol{m}_1 - \boldsymbol{m}_2\|^2 + \mathcal{BW}^2(\bfSigma_1, \bfSigma_2), \end{equation} where the second term, known as the Bures–Wasserstein distance \citep{takatsu_2011_WassersteinGeometryGaussian}, is given by \begin{equation}\label{eq-bures-wasserstein} \mathcal{BW}^2(\bfSigma_1, \bfSigma_2) := \mathrm{Tr} \left( \bfSigma_1 + \bfSigma_2 - 2\left( \bfSigma_2^{1/2} \bfSigma_1 \bfSigma_2^{1/2} \right)^{1/2} \right). \end{equation} The Bures–Wasserstein distance characterizes the discrepancy in second-order structure between Gaussian distributions.

This structure of optimal transport admits natural generalization to statistical manifolds and Riemannian settings. In particular, for probability measures supported on a Riemannian manifold $(\mathcal{M}, g)$, the Wasserstein distance generalizes by replacing the Euclidean cost $d^2(x, y)$ in Equation~\eqref{eq-wasserstein-definition} with the squared geodesic distance $d_{\mathcal{M}}^2(x, y)$. However, closed-form expressions for $\mathcal{W}(\mu, \nu)$ in this setting are rare, and exact computation is often intractable, especially for non-Gaussian distributions or those defined on curved manifolds. In this context, Gaussian approximations on tangent spaces offer a powerful and tractable alternative. This approach is particularly compelling for the vMF distribution, whose high-concentration regime allows accurate approximation by Gaussian distributions supported on the tangent space of the sphere. This insight provides the foundation for our proposed Wasserstein-like dissimilarity measure, which is tailored to the geometric structure of directional data and developed in the next section.
\section{Main}
\subsection{Proposed distance on the space of von Mises-Fisher distributions}

We begin by introducing a novel dissimilarity measure for von Mises–Fisher (vMF) distributions. Let $\bbP_1 = vMF(\bfmu_1, \kappa_1)$ and $\bbP_2 = vMF(\bfmu_2, \kappa_2)$ be two vMF distributions supported on the unit hypersphere $\bbS^{d-1} \subset \bbR^d$, where $\kappa_1, \kappa_2 \in (0, \infty)$. This excludes both uniform distributions corresponding to $\kappa=0$ and Dirac measures, i.e., point masses as $\kappa \to \infty$. We specifically consider the high-concentration regime, where each distribution places a substantial amount of mass around its respective mean direction. From an information-theoretic and geometric perspective, we define the proposed Wasserstein-like dissimilarity $\WL_d(\bbP_1, \bbP_2)$ as \begin{equation}\label{def-WL-distance} \WL_d^2(\bbP_1, \bbP_2) = \arccos^2 (\langle \bfmu_1, \bfmu_2 \rangle) + (d-1)\left(\frac{1}{\sqrt{\kappa_1}} - \frac{1}{\sqrt{\kappa_2}}\right)^2, \end{equation} where $\langle \cdot, \cdot \rangle$ denotes the standard inner product. For notational brevity, we omit the subscript $d$ in $\WL_d$ unless dimensionality needs to be emphasized.

\subsubsection{Motivation}

The proposed dissimilarity quantifies the discrepancy between two vMF distributions parameterized by $(\bfmu_1, \kappa_1)$ and $(\bfmu_2, \kappa_2)$ by separately treating differences in mean direction and concentration, under appropriate geometries. To construct $\WL$, we revisit the high-concentration approximation of the vMF distribution \citep{mardia_2000_DirectionalStatistics}: when $\bfx \sim vMF(\bfmu, \kappa)$ and $\kappa$ is large, the distribution becomes approximately Gaussian in the tangent space of the sphere: \begin{equation*} \bfx \approx \bfmu + \bfv, \quad \text{where} \quad \bfv \sim \calN\left(\boldsymbol{0}, \frac{1}{\kappa}(\bfI_d - \bfmu \bfmu^\top)\right). \end{equation*} Here, $\bfv$ lies in the tangent space $T_{\bfmu} \bbS^{d-1}$, and $\bfI_d - \bfmu \bfmu^\top$ is the projection matrix onto the hyperplane orthogonal to $\bfmu$, asserting that radial variation is negligible for large $\kappa$.

This motivates an approximate decomposition of the Wasserstein distance as \begin{equation}\label{def-vMF-wasserstein-decomposition} \calW^2 (\bbP_1, \bbP_2) \approx \calW_{\bfmu}^2 (\delta_{\bfmu_1}, \delta_{\bfmu_2}) + \calW_{\bfSigma}^2 (\bfSigma_1, \bfSigma_2), \end{equation} where $\delta_{\bfmu_i}$ denotes a Dirac measure at $\bfmu_i$, and $\bfSigma_i$ is the approximate covariance of $\bbP_i$. The first term, $\calW_{\bfmu}^2$, corresponds to the squared geodesic distance on the sphere, which is the squared arc length of the great circle connecting $\bfmu_1$ and $\bfmu_2$:
\begin{equation*}
	\calW_{\bfmu}^2 (\delta_{\bfmu_1}, \delta_{\bfmu_2}) = \arccos^2 (\langle \bfmu_1, \bfmu_2 \rangle).
\end{equation*}
We refer to this component as the geodesic transport term to emphasize its characterization along the shortest path connecting the two points.

We now focus on the second term, which compares the approximate covariances. A natural candidate is the Bures–Wasserstein distance. This formulation, however, requires care in a sense that the matrix $\bfSigma = \frac{1}{\kappa}(\bfI_d - \bfmu \bfmu^\top)$ is rank-deficient, having rank $d-1$. Although its square root is well-defined as $\bfSigma^{1/2} = \frac{1}{\sqrt{\kappa}}(\bfI_d - \bfmu \bfmu^\top)$ due to idempotency of the projection matrix, using rank-deficient matrices in Bures–Wasserstein computations introduces complications. Products like $\bfSigma_2^{1/2} \bfSigma_1 \bfSigma_2^{1/2}$ may be ill-posed or degenerate when the column spaces of the two matrices do not align, potentially leading to discontinuities or infinite transport cost. This compromises desirable properties such as continuity, uniqueness, and computational stability.

To overcome this issue, we exploit the geometry of the sphere via parallel transport, which describes a way of moving vectors from one tangent space to another along a curve on a smooth manifold. This provides an isometric mapping between tangent spaces on a Riemannian manifold. Specifically, parallel transport preserves inner products and hence distances and angles. We use this idea to align the approximate covariances $\bfSigma_1$ and $\bfSigma_2$ by transporting $\bfSigma_1$ from $T_{\bfmu_1} \bbS^{d-1}$ to $T_{\bfmu_2} \bbS^{d-1}$ and comparing the result in a common coordinate system. Let $\bfU \in \bbR^{d \times (d-1)}$ be an orthonormal basis for $T_{\bfmu} \bbS^{d-1}$. The projected covariance in tangent coordinates is
\begin{equation*}
	\bfM = \frac{1}{\kappa} \bfU^\top (\bfI_d - \bfmu \bfmu^\top) \bfU = \frac{1}{\kappa} \bfI_{d-1},
\end{equation*}
since $\bfU^\top \bfmu = \boldsymbol{0}$ and $\bfU^\top (\bfI_d - \bfmu \bfmu^\top) \bfU = \bfI_{d-1}$. Thus, the tangentialized approximate covariance matrix is still isotropic.

For $\bfmu_1, \bfmu_2 \in \bbS^{d-1}$, there exists an orthogonal transformation $\bfQ \in \mathcal{O}(d)$ such that $\bfQ \bfmu_1 = \bfmu_2$ and for $\bfQ(T_{\bfmu_1} \bbS^{d-1}) = T_{\bfmu_2} \bbS^{d-1}$. In other words, such $\bfQ$ is either a rotation or reflection carrying $\bfmu_1$ to $\bfmu_2$ along the shortest path on $\bbS^{d-1}$ and mapping tangent vectors at $\bfmu_1$ to tangent vectors at $\bfmu_2$. Furthermore, it induces a scheme to transport the ambient-space covariance $\bfSigma_1$ under $\bfQ$, yielding $\tilde{\bfSigma}_1 = \bfQ \bfSigma_1 \bfQ^\top$. Since $\bfU_2$ spans the tangent space at $\bfmu_2$, we can represent the transported version of the approximate covariance matrix $\bfSigma_1$ in the tangent coordinates as
\begin{equation}\label{eq-tangent-transformed-approximate-covariance}
	\tilde{\bfM}_1 = \bfU_2^\top \tilde{\bfSigma}_1 \bfU_2 = \bfU_2^\top \bfQ \bfSigma_1 \bfQ^\top \bfU_2.
\end{equation}
If we define $\bfW := \bfU_2^\top \bfQ \bfU_1$, then $\bfW \in \mathcal{O}(d-1)$ since $\bfW^\top \bfW = \bfW \bfW^\top  = \bfI_{d-1}$. Hence, Equation \eqref{eq-tangent-transformed-approximate-covariance} can be simplified as 
\begin{equation}\label{eq-tangent-transformed-approximate-covariance-simplified}
	\tilde{\bfM}_1 = (\bfU_2^\top \bfQ \bfU_1) (\bfU_1^\top \bfSigma_1  \bfU_1)(\bfU_1^\top \bfQ^\top \bfU_2) =  \bfW \bfM_1 \bfW^\top.
\end{equation}
We previously observed that $\bfM_1$ is an isotropic covariance matrix, hence
\begin{equation}\label{eq-two-tangentialized-matrices}
\tilde{\bfM}_1 =  \bfW \left(
\frac{1}{\kappa_1} \bfI_{d-1} 
\right) \bfW^\top = \frac{1}{\kappa_1} \bfI_{d-1}\quad\textrm{and}\quad \bfM_2 = \frac{1}{\kappa_2}\bfI_{d-1}.
\end{equation}
Since $\tilde{\bfM}_1$ and $\bfM_2$ are both full-rank and aligned in a common tangent space, we propose to use the Bures-Wasserstein distance between these two matrices as a discrepancy between the two approximate covariances in the ambient space $\bbR^d$:
\begin{equation}\label{eq-W-EQ-covariance}
	W_{\bfSigma}^2 (\bfSigma_1, \bfSigma_2) := \mathcal{BW}^2 (\tilde{\bfM}_1, \bfM_2) = (d-1) \left(\frac{1}{\sqrt{\kappa_1}} - \frac{1}{\sqrt{\kappa_2}} \right)^2.
\end{equation}
The last equality follows from applying Equation  \eqref{eq-bures-wasserstein} to scalar multiples of the identity matrices:
\begin{align*}
\mathcal{BW}^2 (\tilde{\bfM}_1, \bfM_2) &= \mathcal{BW}^2 \left(\frac{1}{\kappa_1} \bfI_{d-1}, \frac{1}{\kappa_2} \bfI_{d-1}\right) \\ 
&= \mathrm{Tr}\left(
\frac{1}{\kappa_1} \bfI_{d-1} + \frac{1}{\kappa_2} \bfI_{d-1}
- 2 \left( 
\frac{1}{\sqrt{\kappa_1}} \bfI_{d-1} \frac{1}{\kappa_2} \bfI_{d-1} \frac{1}{\sqrt{\kappa_1}} \bfI_{d-1}
\right)^{1/2}
\right) \\
	&=\frac{1}{\kappa_1} \mathrm{Tr}\left(
\bfI_{d-1}
\right)
+ \frac{1}{\kappa_2}\mathrm{Tr}\left(
 \bfI_{d-1}
\right)
- \frac{2}{\sqrt{\kappa_1 \kappa_2}} \mathrm{Tr}\left(
 \bfI_{d-1}
\right)\\
&= \left(
\frac{1}{\sqrt{\kappa_1}} - \frac{1}{\sqrt{\kappa_2}}
\right)^2 \mathrm{Tr}(\bfI_{d-1}) \\
&= (d-1) \left(
\frac{1}{\sqrt{\kappa_1}} - \frac{1}{\sqrt{\kappa_2}}
\right)^2.
\end{align*}
This completes the construction of the Wasserstein-like dissimilarity $\WL$ as defined in Equation~\eqref{def-WL-distance}, where the second term that captures differences in concentration is referred to as the variance transport component.

\subsubsection{Theoretical properties}

Given the novel dissimilarity $\WL$ for comparing vMF distributions, it is imperative to verify its theoretical validity, starting with the question of well-posedness.

\begin{theorem}[Well-posedness]\label{theorem-wellposed}
	Let $\lbrace (\bfmu_i,\kappa_i)\rbrace_{i\in \mathcal{I}}$ be a collection of vMF distributions with $\kappa_i \in (0, \infty)$. Then the $\WL$ dissimilarity is: (1) well-defined, (2) continuous, (3) topologically consistent to induce a meaningful topology, and (4) well-behaved in extreme cases of the concentration parameters.
\end{theorem}
\begin{proof}[Proof of Theorem \eqref{theorem-wellposed}]
	First, the geodesic term $\arccos^2(\langle \bfmu_1, \bfmu_2 \rangle)$ is well-defined for all unit-norm mean directions since their dot product lies in $[-1,1]$. The assumption $\kappa_i \in (0, \infty)$ ensures that the variance transport term remains finite, excluding degenerate cases such as uniform and Dirac measures.
	
	Second, recall that a function is continuous if small perturbations in inputs induce small changes in outputs. The geodesic term depends continuously on the inner product $\langle \bfmu_1, \bfmu_2 \rangle$, which is itself continuous in $\bfmu_1$ and $\bfmu_2$. The variance transport term is continuous in $\kappa_1, \kappa_2$ since the map $\kappa \mapsto 1/\sqrt{\kappa}$ is continuous on $(0,\infty)$, and operations like subtraction, squaring, and scalar multiplication preserve continuity. Since $\WL$ is a sum of continuous functions, it is globally continuous.
	
	Third, the geodesic term induces the standard topology on $\bbS^{d-1}$, while the variance transport term corresponds to a Euclidean distance in a transformed coordinate system on $\bbRpos$. This transformation is a diffeomorphism and thus preserves the topology of the positive real line. Their squared sum defines a valid metric on the product space, thus inducing a meaningful topology that respects both angular and concentration structure.
	
	Lastly, we analyze two limiting cases of the concentraion parameters. As $\kappa_1, \kappa_2 \to 0$, both distributions converge to the uniform distribution and the variance transport term diverges, dominating the metric. This reflects the fact that small changes in low-concentration regimes can result in large distributional differences. When $\kappa_1, \kappa_2 \to \infty$, both distributions concentrate into point masses, and the variance transport term vanishes. In this limit, the dissimilarity smoothly reduces to the geodesic distance between mean directions. This concludes the proof.
\end{proof}

The topological consistency established above suggests that $\WL$ defines a distance over the product parameter space of vMF distributions. We next examine the geometry underlying the variance transport term by studying the metric space structure it induces on $\bbRpos$.

\begin{lemma}\label{lemma-positive-real}
	On the set of strictly positive real numbers, define a distance function $d:\bbRpos \times \bbRpos \rightarrow \mathbb{R}_{\geq 0}$ by 
	\begin{equation*}
		d(x,y) = c \left\lvert
		\frac{1}{\sqrt{x}} - \frac{1}{\sqrt{y}}
		\right\rvert,
	\end{equation*}
	for some constant $c>0$. Then, $(\bbRpos, d)$ is a one-dimensional Riemannian manifold with the metric tensor
	\begin{equation*}
		g(x) = \frac{c^2}{4}x^{-3} dx^2.
	\end{equation*}
\end{lemma}
\begin{proof}[Proof of Lemma \ref{lemma-positive-real}]
	Consider the diffeomorphism $f:\bbRpos \rightarrow (0,\infty)$ given by $f(x) = 1/\sqrt{x}$. The squared differential of $f$ is 
	\begin{equation*}
		(df)^2 = (f'(x)dx)^2 = \left(-\frac{1}{2}x^{-3/2}\right)^2 dx^2 = \frac{1}{4}x^{-3}dx^2.
	\end{equation*}
	Multiplying by  the constant $c^2$, the pullback metric on $\bbRpos$ becomes
	\begin{equation*}
		g(x) = c^2 \cdot \frac{1}{4} x^{-3} dx^2 = \frac{c^2}{4}x^{-3} dx^2.
	\end{equation*}
	To verify that this metric induces the claimed distance, compute the length element  $ds$ associated with the metric tensor $g(x)$:
	\begin{equation*}
		ds = \sqrt{g(x)} dx = \frac{c}{2} x^{-3/2} dx.
	\end{equation*}
	Hence, the distance between $x,y \in \bbRpos$ is 
	\begin{equation*}
		d(x,y) = \left\lvert \int_x^y ds \right\rvert =  \left\lvert 
		\int_{x}^y \frac{c}{2} z^{-3/2}dz
		\right\rvert = \left\lvert -c  \left(y^{-1/2} - x^{-1/2}\right)\right\rvert = c \left\lvert \frac{1}{\sqrt{x}} - \frac{1}{\sqrt{y}} \right\rvert,
	\end{equation*}
	which is consistent to the above distance function. 
	
	Since $\bbRpos$ is a smooth manifold and $g(x)$ is a smoothly varying, positive-definite symmetric bilinear form on the tangent space, $(\bbR,d)$ is indeed a one-dimensional Riemannian manifold and the distance function $d(\cdot,\cdot)$ corresponds to the geodesic distance induced by this Riemannian metric $g(x)$.
\end{proof}

By setting the constant $c = \sqrt{d-1}$, Lemma \ref{lemma-positive-real} shows that the variance transport component in the $\WL$ dissimilarity corresponds to a geodesic distance on $\bbRpos$ under a well-defined Riemannian metric. Together with the geodesic transport term, this leads to the following result:

\begin{theorem}\label{theorem-metric}
	The $\WL$ dissimilarity is a valid distance metric on the space of non-degenerate vMF distributions.
\end{theorem}
\begin{proof}[Proof of Theorem \ref{theorem-metric}]
The parameters of a vMF distribution lie in the product manifold $\Theta := \bbS^{d-1} \times \bbRpos$. Both $\bbS^{d-1}$ and $\bbRpos$ are Riemannian manifolds with geodesic distances given by the arc length and Lemma~\ref{lemma-positive-real}, respectively.

By a standard result in Riemannian geometry \citep{lee_1997_RiemannianManifoldsIntroduction}, the product manifold $M^\otimes := M_1 \times M_2$ of two Riemannian manifolds $(M_1, g_1)$ and $(M_2, g_2)$ inherits a natural product metric $g^\otimes$ and corresponding geodesic distance \begin{equation}\label{eq-geodesic-distance-on-product-manifold} d^\otimes((p_1, p_2), (q_1, q_2)) = \sqrt{d_{M_1}(p_1, q_1)^2 + d_{M_2}(p_2, q_2)^2}. \end{equation} Here, $p_i, q_i \in M_i$ and $d_{M_i}$ denotes the geodesic distance on $M_i$ for $i=1,2$. In our case, the $\WL$ dissimilarity exactly matches this form, with $M_1 = \bbS^{d-1}$ and $M_2 = \bbRpos$. Therefore, $\WL$ satisfies the axioms of a metric and defines a valid distance on $\Theta$.
\end{proof}

The space of non-degenerate vMF distributions can thus be viewed as a Riemannian manifold equipped with a product structure over its parameters. Under the high-concentration Gaussian approximation, the $\WL$ dissimilarity corresponds to the geodesic distance on this manifold. Hence, we will treat $\WL$ as a valid distance for vMF distributions in all subsequent sections.

\subsection{Barycenter}

The $\WL$ distance in Equation~\eqref{def-WL-distance} endows the space of vMF distributions with a Riemannian manifold structure. Given a collection of vMF distributions, this induced geometry allows us to define notions of central tendency, analogous to means in Euclidean space, through the concept of barycenters \citep{afsari_2011_Riemannian$L_p$Center}. Specifically, we define the geometric mean of a finite set of vMF distributions $\lbrace \bbP_i\rbrace_{i=1}^n$ as the minimizer of the weighted sum of squared $\WL$ distances:
\begin{equation}\label{barycenter-measure-form}
	F(\bbP) = \sum_{i=1}^n w_i \WL^2 (\bbP, \bbP_i),
\end{equation}
where $\bfw = [w_1, \ldots, w_n] \in \Delta_{n-1}$ is a probability weight vector. Equivalently, expressing the functional in terms of the distribution parameters $(\bfmu, \kappa)$ yields
\begin{equation}\label{barycenter-parameter-form}
	F(\bfmu,\kappa) = \sum_{i=1}^n w_i \left\lbrace
	\arccos^2(\langle \bfmu, \bfmu_i\rangle) + 
	(d-1)\left(\frac{1}{\sqrt{\kappa}} - \frac{1}{\sqrt{\kappa_1}}\right)^2\right\rbrace.
\end{equation}
We refer to any minimizer $(\hat{\bfmu}, \hat{\kappa})$ of this functional as a barycenter. This terminology is justified by the analogy to the Wasserstein barycenter, which defines a geometric mean of probability measures under optimal transport \citep{agueh_2011_BarycentersWassersteinSpace}. In our setting, the minimizer corresponds to a unique vMF distribution that best summarizes the given collection in terms of the $\WL$ geometry.

The barycenter functional is well-defined, which follows directly from arguments similar to those in Theorem~\ref{theorem-wellposed}. While the existence of a minimizer is guaranteed due to compactness and continuity, the question of uniqueness is more subtle. We address this in the following result.

\begin{theorem}\label{theorem-barycenter-uniqueness}
	Let $\bbP_1, \ldots, \bbP_n$ be a collection of vMF distributions parameterized by $(\bfmu_i, \kappa_i)$ for $i = 1, \ldots, n$. If the mean directions $\bfmu_1, \ldots, \bfmu_n$ are all contained within an open geodesic ball of radius $\pi/2$ centered at some $\bfx \in \bbS^{d-1}$, then the barycenter functional admits a unique minimizer $(\hat{\bfmu}, \hat{\kappa})$.
\end{theorem}
\begin{proof}[Proof of Theorem \ref{theorem-barycenter-uniqueness}]
	The barycenter functional $F(\bfmu, \kappa)$ from Equation~\eqref{barycenter-parameter-form} can be decomposed into two independent terms: 
	\begin{equation}\label{eq-barycenter-uniqueness}
		F(\bfmu, \kappa) = \sum_{i=1}^n w_i \arccos^2 (\langle \bfmu, \bfmu_i \rangle) + (d-1) \sum_{i=1}^n w_i \left(\frac{1}{\sqrt{\kappa}} - \frac{1}{\sqrt{\kappa_i}}\right)^2.
	\end{equation}
	
	The first term corresponds to a weighted sum of squared geodesic distances on the unit hypersphere. This is equivalent to computing the weighted \Frechet mean of the unit vectors $\bfmu_i's$ with respect to the canonical metric on $\bbS^{d-1}$, which is a well-studied problem on a general Riemannian manifold $\calM$ \citep{kendall_1990_ProbabilityConvexityHarmonic}. Consider a collection of points $x_1, \ldots, x_n \in \calM$, which constitutes an empirical measure. If it is supported in an open geodesic ball centered at some $x \in \calM$ with a radius $r^*/2$ for $r^* = \textrm{min} \lbrace \textrm{inj}(\calM), \pi/\sqrt{C^*}\rbrace $ where $\textrm{inj}(\mathcal{M})$ is the injectivity radius and $C^*$ is the least upper bound of sectional curvature of $\calM$, then a unique sample mean exists within the ball \citep{bhattacharya_2012_NonparametricInferenceManifolds}. For the unit hypersphere $\bbS^{d-1}$, we have $\mathrm{inj}(\bbS^{d-1}) = \pi$ and $C^* = 1$ as it has constant sectional curvature. Therefore, the uniqueness of the minimizer $\hat{\bfmu}$ is guaranteed as long as the support lies within an open ball of radius less than $\pi/2$, which is precisely the assumption in the theorem.
	
	The second term in Equation~\eqref{eq-barycenter-uniqueness} involves minimizing a strictly convex function of $1/\sqrt{\kappa}$. Since the map $\kappa \mapsto 1/\sqrt{\kappa}$ is a diffeomorphism from $(0,\infty)$ onto $(0,\infty)$, and the squared difference function is strictly convex, the weighted sum is strictly convex as well. Thus, there exists a unique minimizer $\hat{\kappa} > 0$. Combining the uniqueness of both components yields the uniqueness of the barycenter $(\hat{\bfmu}, \hat{\kappa})$.
\end{proof}

The computation of the vMF barycenter naturally decomposes into two subproblems, owing to the additive form of the functional in Equation~\eqref{barycenter-parameter-form}. According to Theorem~\ref{theorem-barycenter-uniqueness}, the optimal concentration parameter admits a closed-form expression: 
\begin{equation}\label{solution-optimal-concentration}
	\hat{\kappa} = \left(\sum_{i=1}^n \frac{w_i}{\sqrt{\kappa_i}}\right)^{-2}.
\end{equation} 
Since the concentration parameter is analogous to the reciprocal of a variance, this formula may be interpreted as a type of harmonic mean of the standard deviations. Notably, $\hat{\kappa}$ is dimension-agnostic: it depends only on the concentration parameters and weights, not on the intrinsic dimensionality of the hypersphere.

In contrast, estimating the barycenter of the mean directions $\hat{\bfmu}$ requires a numerical procedure on the sphere $\bbS^{d-1}$. This falls within the domain of optimization on Riemannian manifolds \citep{absil_2008_OptimizationAlgorithmsMatrix}. Among various available techniques, we employ a first-order Riemannian gradient descent method, which has proven effective for computing weighted \Frechet means on spherical domains \citep{you_2022_ParameterEstimationModelbased}.

For simplicity, consider the minimization problem: 
\begin{equation}\label{eq-subproblem-frechet-mean}
	G(\bfmu) = \sum_{i=1}^n w_i d_{geo}^2 (\bfmu, \bfmu_i),
\end{equation}
where $d_{\mathrm{geo}}(\bfx, \bfy) = \arccos(\langle \bfx, \bfy \rangle)$ denotes the geodesic distance for  $\bfx, \bfy \in \bbS^{d-1}$. Analogous to Euclidean gradient descent, the Riemannian version iteratively updates the current estimate along the direction of steepest descent on the manifold. However, since the gradient lies in the tangent space, additional steps are required to map between the manifold and its tangent space. These are accomplished using the exponential and logarithmic maps. For $\bfx, \bfy \in \bbS^{d-1}$ and a tangent vector $\bfv \in T_{\bfx}\bbS^{d-1}$, the maps are defined as \begin{equation}\label{eq-sphere-exponential-logarithmic-maps}
	\Exp_{\bfx} (\bfv) = \cos (\| \bfv \|) \bfx + \frac{\sin \|\bfv\|}{\|\bfv\|}\bfv \quad \textrm{and}\quad \Log_{\bfx} (\bfy) = \frac{d_{geo} (\bfx,\bfy)}{\| \textrm{Proj}_{\bfx}(\bfy-\bfx)\|} \textrm{Proj}_{\bfx} (\bfy-\bfx),
\end{equation} where the projection operator $\mathrm{Proj}{\bfx}(\bfz) = \bfz - \langle \bfx, \bfz \rangle \bfx$ maps a vector $\bfz \in \bbR^d$ onto the tangent space $T{\bfx} \bbS^{d-1}$. Then, the Riemannian gradient descent proceeds by the following update rule given an initial starting point $\bfmu^{(0)} \in \bbS^{d-1}$:
\begin{equation}\label{eq-geodesic-mean-update}
	\bfmu^{(t+1)} \leftarrow \Exp_{\bfmu^{(t)}} \left(2 \alpha^{(t)} \sum_{i=1}^n w_i \Log_{\bfmu^{(t)}} (\bfmu_i)\right),
\end{equation}
where $\alpha^{(t)} > 0$ is the step size of gradient marching at iteration $t$. 

We close this section by referring to several algorthmic components in barycenter computation. Following the recommendations in \citet{you_2022_ParameterEstimationModelbased}, we initialize the procedure using the normalized weighted Euclidean average
\begin{equation*}
\bfmu^{(0)} = \frac{\sum_{i=1}^n w_i \bfmu_i}{\|\sum_{i=1}^n w_i \bfmu_i\|}.
\end{equation*}
This choice corresponds to the extrinsic mean, which in theory lies close to the intrinsic \Frechet mean under the support condition of Theorem~\ref{theorem-barycenter-uniqueness}. For the step size, we propose using a fixed value $\alpha^{(t)} = 0.25$, as suggested in \citep{hauberg_2018_DirectionalStatisticsSpherical}. This constant-step strategy avoids costly line search procedures, making it computationally efficient in high-dimensional or large-scale settings, albeit at the cost of potentially requiring more iterations for convergence. The stopping rule can be based on the change in iterates to terminate the algorithm when  $\| \bfmu^{(t+1)} - \bfmu^{(t)} \| < \epsilon$ for a small threshold $\epsilon > 0$. Since $\bbS^{d-1}$ is a complete Riemannian manifold, every Cauchy sequence converges to a point. Moreover, as iterates approach convergence, the geodesic distance between points on the manifold converges to their Euclidean distance asymptotically in the ambient space. Thus, measuring convergence in $\ell_2$ norm is almost consistent with using geodesic distance, and equivalent up to first order to standard stopping criteria based on gradient magnitude in manifold optimization \citep{boumal_2023_IntroductionOptimizationSmooth}.

\section{vMF Mixture Model Reduction}


\subsection{Greedy method}

The greedy method for mixture model reduction is an iterative strategy that reduces the number of components by successively merging the most similar pairs. This approach relies on a well-defined dissimilarity measure to guide local decisions at each iteration, without explicitly considering the global structure of the mixture.

We propose to use the $\WL$ distance as a geometrically principled dissimilarity for constructing a greedy algorithm. At iteration $t$, the current vMF mixture consists of $S:=K-t+1$ components, represented as  $\bbQ^{(t)} = \sum_{k=1}^S \alpha_k \bbP_k$,
where $\alpha_k$'s are the mixture weights. The iteration begins with a \textsf{search} step to identify the pair of components with minimal discrepancy under the $\WL$ distance: 
\begin{equation}\label{eq-minimizing-index-distance}
	(i^*, j^*) = \underset{(i,j)}{\argmin} \, \WL(\bbP_i, \bbP_j) \,\, \textrm{such that } (i,j) \in [S]\times [S], ~ i\neq j.
\end{equation}
The subsequent \textsf{merge} step replaces the identified components $\bbP_{i^*}, \bbP_{j^*}$ with a single merged component $\bbP_{*}$, and updates the corresponding weights by setting $\alpha_{*} = \alpha_{i^*} + \alpha_{j^*}$, thus preserving the mixture property that weights sum to one. A natural choice for the merged component $\bbP_{*}$ is the barycenter of $\bbP_{i^*}$ and $\bbP_{j^*}$:
\begin{equation}\label{eq-barycenter-closest}
	\bbP_{*} = \underset{\bbP}{\argmin}\, \left[\tilde{\alpha}_{i} \WL^2 (\bbP, \bbP_{i^*}) 
	+
	\tilde{\alpha}_{j} \WL^2 (\bbP, \bbP_{j^*})
	\right].
\end{equation} 
Here, the normalized weights are defined as $\tilde{\alpha}_i = \alpha_{i^*}/(\alpha_{i^*} + \alpha_{j^*})$ and $\tilde{\alpha}_j = \alpha_{j^*}/(\alpha_{i^*} + \alpha_{j^*})$ so that $\tilde{\alpha}_i + \tilde{\alpha}_j = 1$. This form aligns exactly with the weighted barycenter formulation in Equation \eqref{barycenter-measure-form} and its parametric variant in Equation \eqref{barycenter-parameter-form}. All the other mixture components and weights remain unchanged in this update. The process is repeated until the total number of components is reduced to the desired target size $K' < K$.

\subsection{Partitional method}

An alternative to the greedy approach is the partitional method, which reduces the mixture model in a single pass by partitioning the components into groups and summarizing each group with a barycenter. Unlike the greedy strategy, which is iterative and local, the partitional method clusters all components at once, ignoring their weights during the clustering phase.

This approach is partly motivated by the pigeonhole principle: if a mixture contains more components than necessary, redundancy must exist, implying that a smaller number of well-chosen components can plausibly represent the same distribution.

The method begins by executing the \textsf{search} step in a single stage: partition the $K$ components into $K'$ clusters using a clustering algorithm, where $K' < K$ is the target number of components. We consider two clustering methods for this task. 
\begin{itemize}
	\item Hierarchical clustering with single linkage, which iteratively merges the closest pairs of components based on minimum pairwise dissimilarity \citep{johnson_1967_HierarchicalClusteringSchemes, gower_1969_MinimumSpanningTrees}. This method constructs a dendrogram whose cuts yield partitions for any desired number of clusters $K'$. The single linkage criterion aligns with ideas from topological data analysis (TDA) \citep{carlsson_2004_PersistenceBarcodesShapes, chazal_2021_IntroductionTopologicalData}, where a filtration parameter $\epsilon$ is used to form simplicial complexes and track merging of connected components, corresponding to 0-dimensional persistent homology.
	\item $k$-medoids clustering, which is a robust alternative to $k$-means that minimizes within-cluster dissimilarity using actual data points as representatives \citep{kaufman_1990_PartitioningMedoidsProgram}. It requires only pairwise distances and avoids computing cluster means or barycenters. Its robustness to noise and outliers makes it particularly suitable for modeling with mixture components \citep{huber_1981_RobustStatistics}.
\end{itemize}

Once the clustering is complete, we obtain a partition ${\calI_1, \ldots, \calI_{K'}}$ of the indices $[K]$, where each $\calI_k$ corresponds to a group of vMF components to be merged. The \textsf{merge} step then summarizes each cluster via a barycenter. Without loss of generality, suppose the $k$-th cluster consists of $n_k$ vMF distributions ${\bbP_1, \ldots, \bbP_{n_k}}$ with original weights ${\alpha_1, \ldots, \alpha_{n_k}}$. Then, the barycenter component representing this cluster is
\begin{equation*}
	\bbP_*^{(k)} = \underset{\bbP}{\argmin} \sum_{i=1}^{n_k} \tilde{\alpha}_i \WL^2 (\bbP, \bbP_i),
\end{equation*}
where the normalized weights are $\tilde{\alpha}_i = \frac{\alpha_i}{\sum_{j=1}^{n_k} \alpha_j}$ for $i \in [n_k]$. The total weight assigned to the cluster is the sum of its original weights, i.e., $\alpha^{(k)} = \sum_{i=1}^{n_k} \alpha_i$. That is, contribution of all components in the cluster is imposed onto the barycenter. This procedure is repeated for all clusters, resulting in a reduced mixture model characterized by weights and barycentric vMF components $\lbrace (\alpha^{(i)}, \bbP_*^{(i)})\rbrace$ for $k \in [K']$.

\section{Examples}

In this section, we demonstrate the distinctive properties of the $\WL$ distance as a metric on the space of non-degenerate vMF distributions. In particular, we illustrate its effectiveness as a discriminative tool and its utility in mixture model reduction through both simulated and real-world examples.

\subsection{Simulated Examples}

We begin with a simulation study designed to evaluate the capacity of the $\WL$ distance to distinguish between vMF distributions with varying parameters. Specifically, we construct four distinct types of vMF distributions by varying both their mean direction and concentration parameter.

In the case where $d = 2$, the unit circle $\bbS^1$ admits a one-to-one mapping with the angular coordinate $\theta$ under the polar coordinate system, with points represented as $[x, y] = [\cos(\theta), \sin(\theta)]$. We consider two directional regimes corresponding to ``north'' and ``south'' orientations. The directional parameters are sampled from uniform distributions: $\theta \sim \text{Uniform}(15\pi/8, 17\pi/8)$ for the north group and $\theta \sim \text{Uniform}(7\pi/8, 9\pi/8)$ for the south group. Similarly, we define two regimes for concentration by sampling $\kappa \sim \text{Uniform}(0.9, 1.1)$ to represent low-concentration with broad dispersion and $\kappa \sim \text{Uniform}(9.9, 10.1)$ for high-concentration, tight-clustering  vMF distributions. This results in a $2 \times 2$ factorial design, yielding four distinct classes of distributions: \textsf{north-low}, \textsf{north-high}, \textsf{south-low}, and \textsf{south-high}. Random samples from each of these four types are generated and visualized in Figure~\ref{fig:sim1_four_types}, which displays representative densities on $\bbS^1$ under each condition.

\begin{figure}[ht]
	\centering
	\begin{minipage}{0.24\textwidth}
		\centering
		\includegraphics[width=\textwidth]{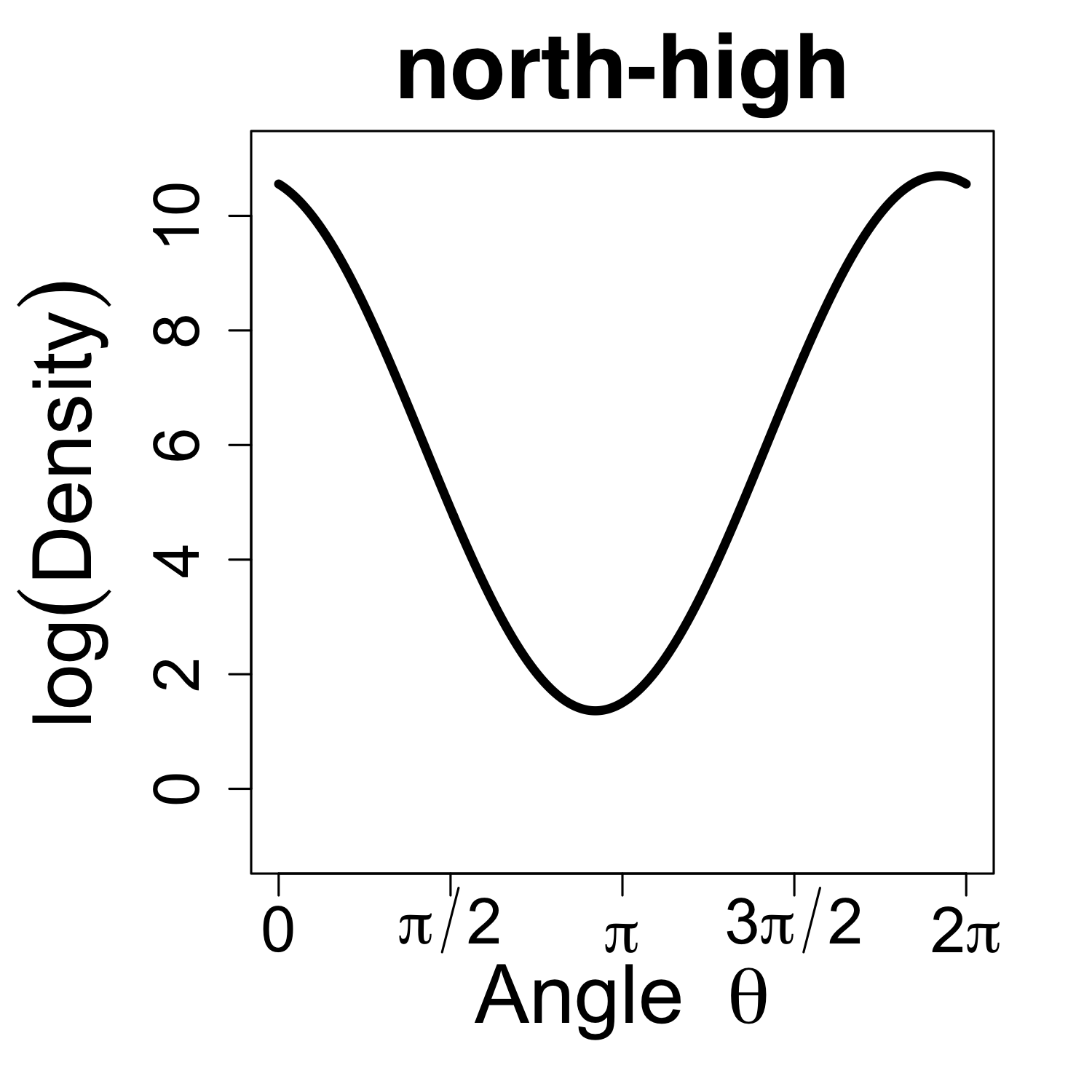}
	\end{minipage}
	\begin{minipage}{0.24\textwidth}
		\centering
		\includegraphics[width=\textwidth]{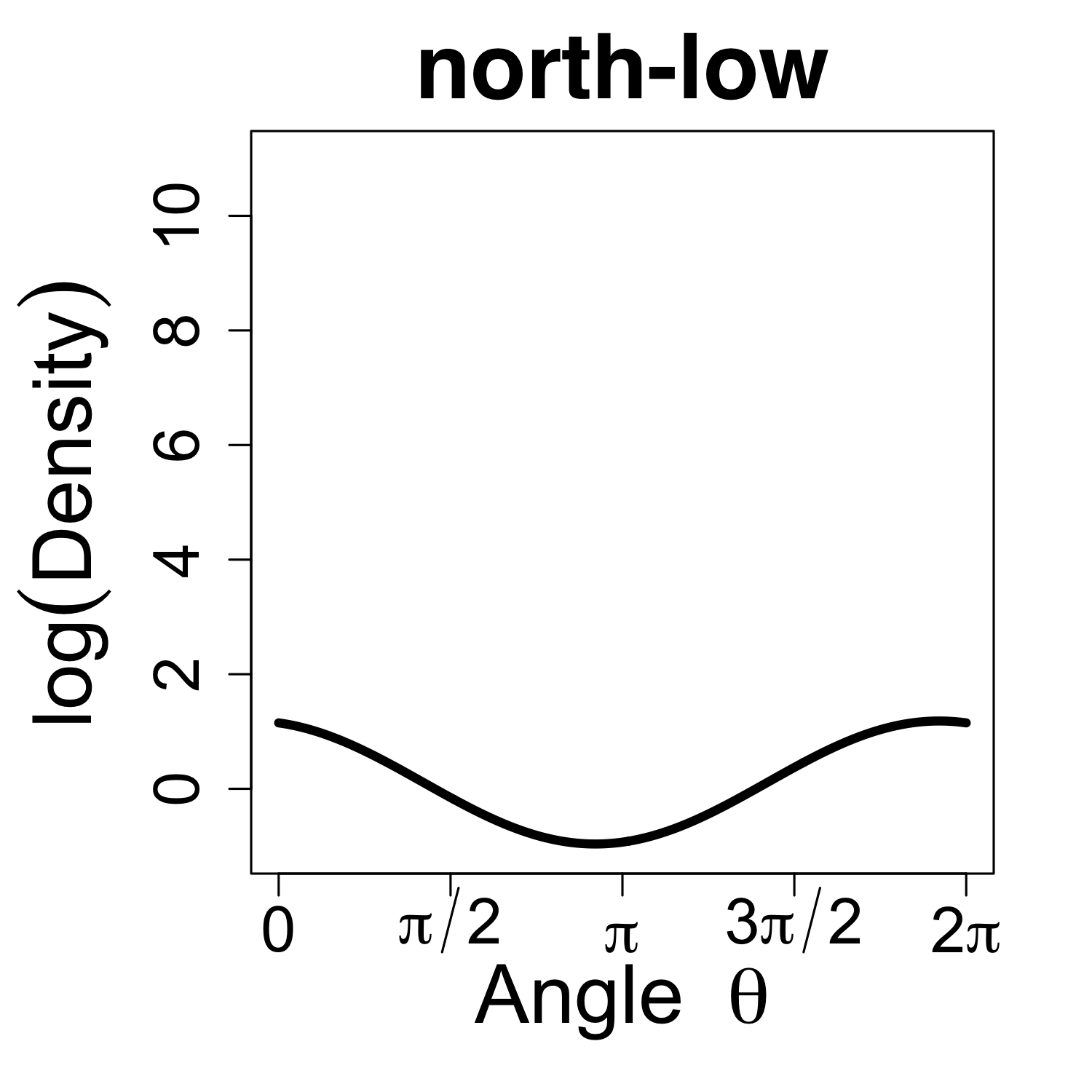}
	\end{minipage}
	\begin{minipage}{0.24\textwidth}
		\centering
		\includegraphics[width=\textwidth]{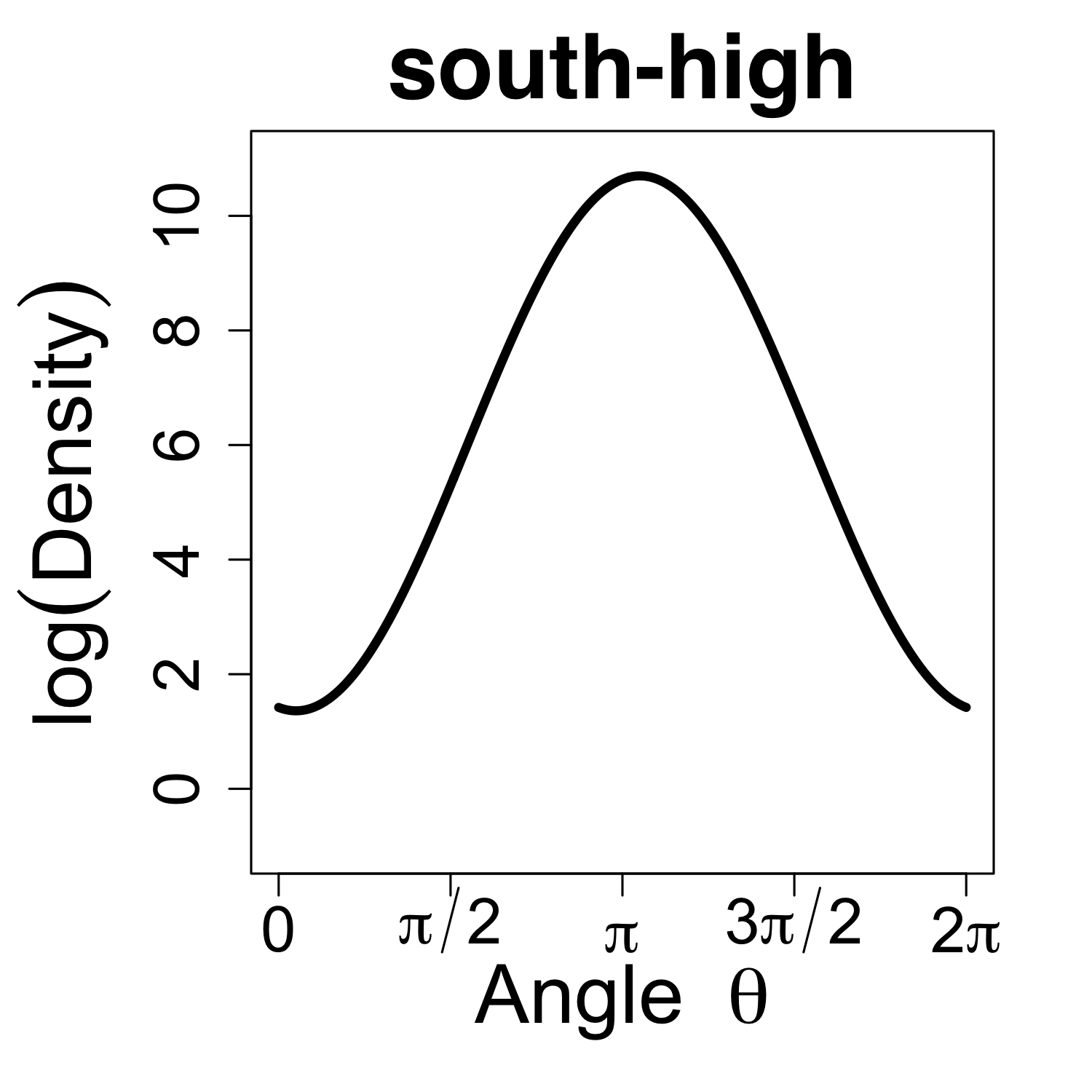}
	\end{minipage}
	\begin{minipage}{0.24\textwidth}
		\centering
		\includegraphics[width=\textwidth]{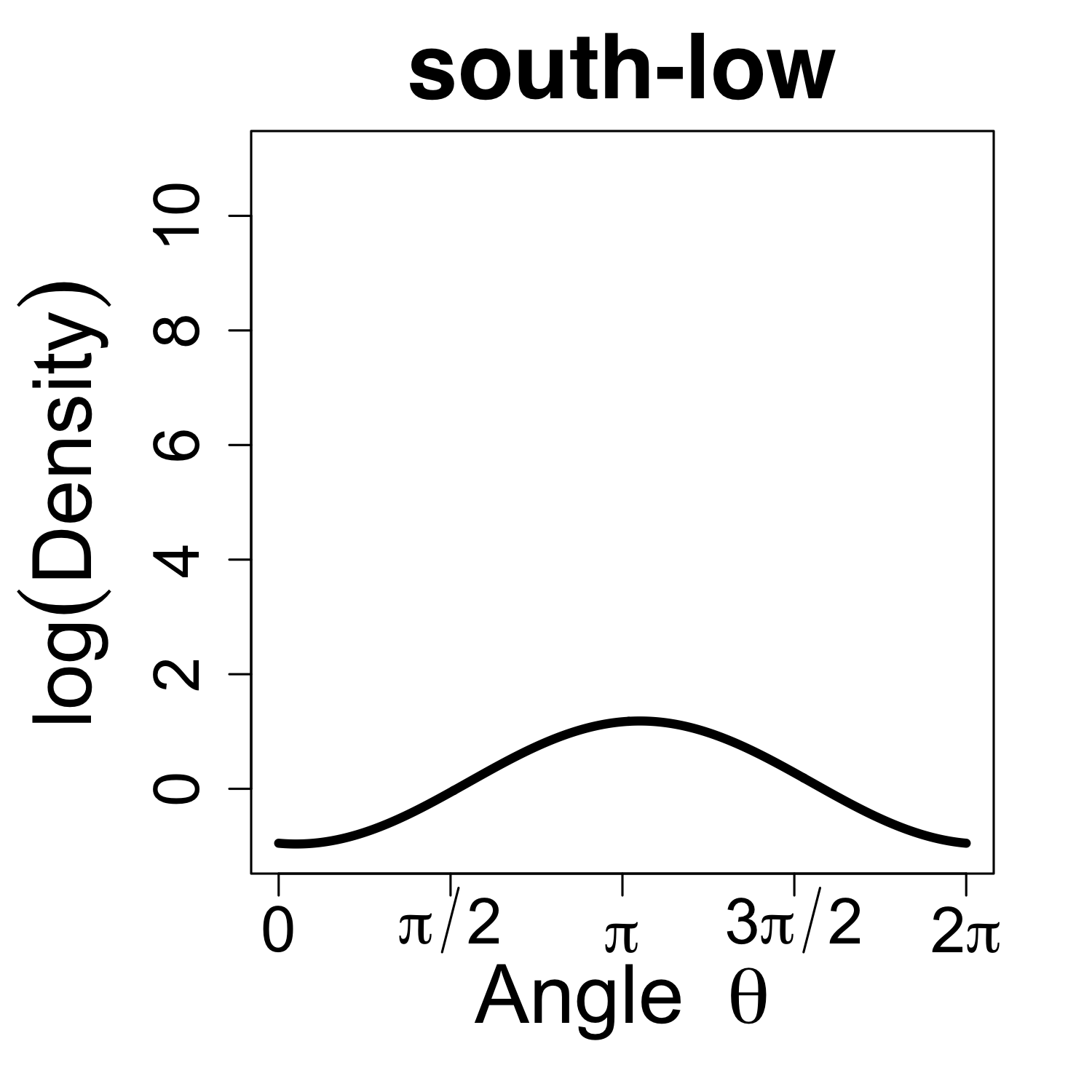}
	\end{minipage}
	\caption{Representative densities of four von Mises–Fisher (vMF) distribution types formed by combinations of perturbed mean directions and concentration parameters.}
	\label{fig:sim1_four_types}
\end{figure}

For each of the four distribution types, we generated a random sample of 100 vMF parameter pairs, resulting in a total of 400 distinct vMF distributions. We computed all pairwise dissimilarities using the proposed $\WL$ distance. For comparison, we also evaluated the standard $L_2$ distance between the corresponding density functions. Given two vMF distributions $\bbP_1$ and $\bbP_2$ with densities $f_1$ and $f_2$, the $L_2$ distance is defined as 
\begin{equation*}
	L_2 (\bbP_1, \bbP_2) = \left(\int_{\bbS^d} \left( f_1 (\bfx) - f_2 (\bfx) \right)^2 d\bfx\right)^{1/2}.
\end{equation*}
As this integral lacks a closed-form solution, we approximated it using Monte Carlo integration \citep{metropolis_1949_MonteCarloMethod}. Specifically, we drew $10^3$ samples uniformly from the unit hypersphere, with the sample size chosen empirically to ensure stable pairwise distance estimates. To visualize the dissimilarity structure, we applied multi-dimensional scaling (MDS) \citep{torgerson_1952_MultidimensionalScalingTheory} to each of the distance matrices, embedding the 400 vMF distributions into $\mathbb{R}^2$ for interpretation.

\begin{figure}[ht]
	\centering
	\begin{minipage}{0.325\textwidth}
		\centering
		\includegraphics[width=\textwidth]{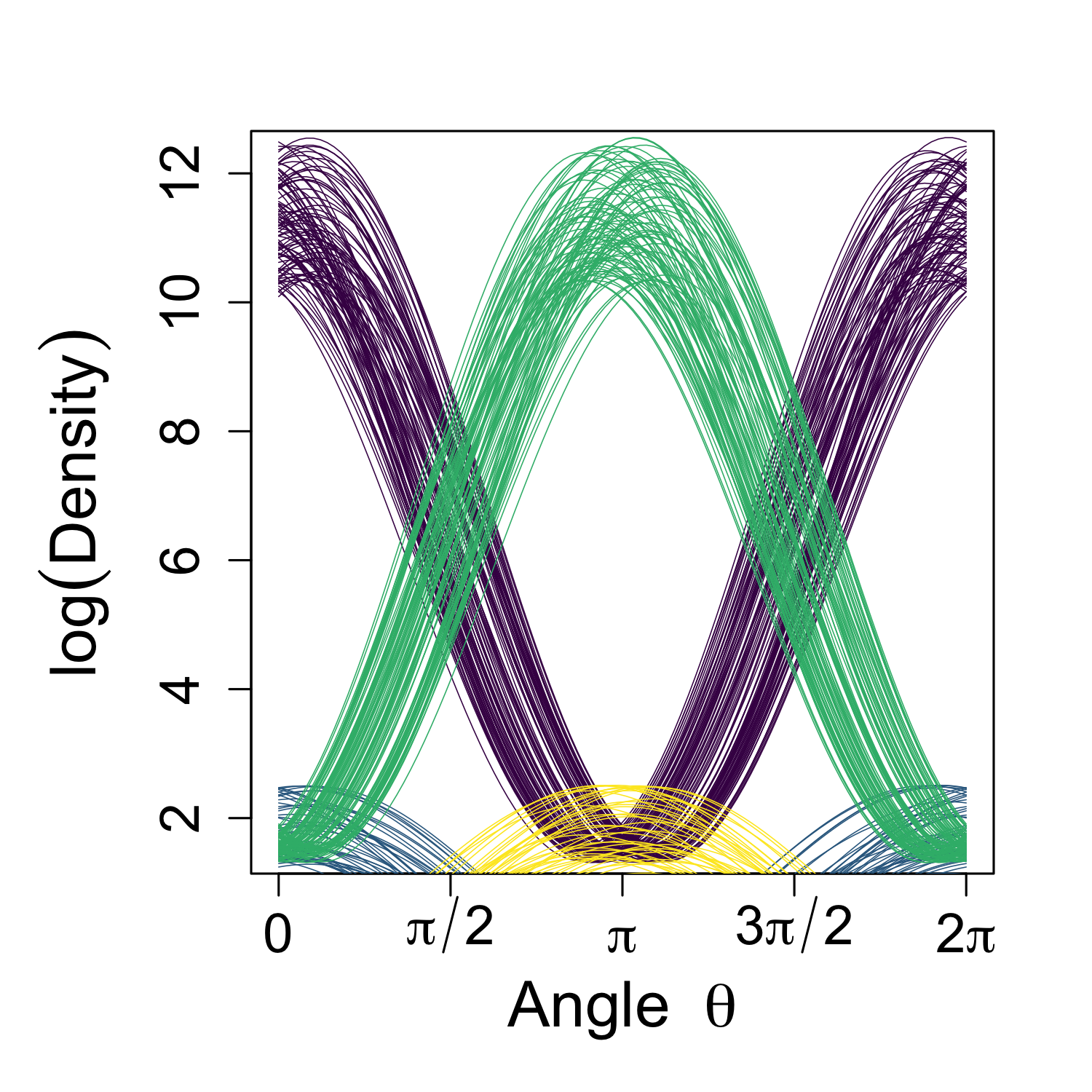}
	\end{minipage}
	\begin{minipage}{0.325\textwidth}
		\centering
		\includegraphics[width=\textwidth]{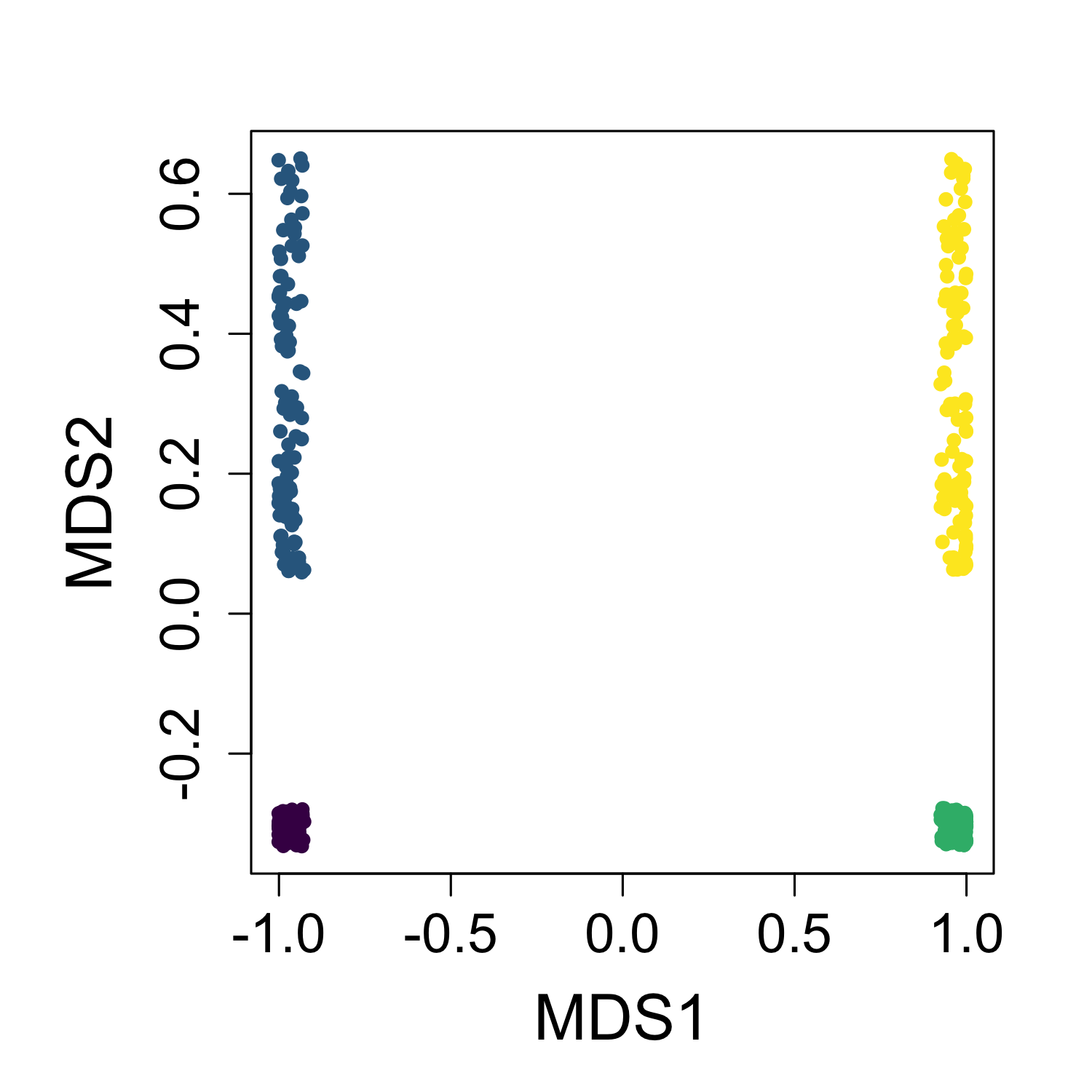}
	\end{minipage}
	\begin{minipage}{0.325\textwidth}
		\centering
		\includegraphics[width=\textwidth]{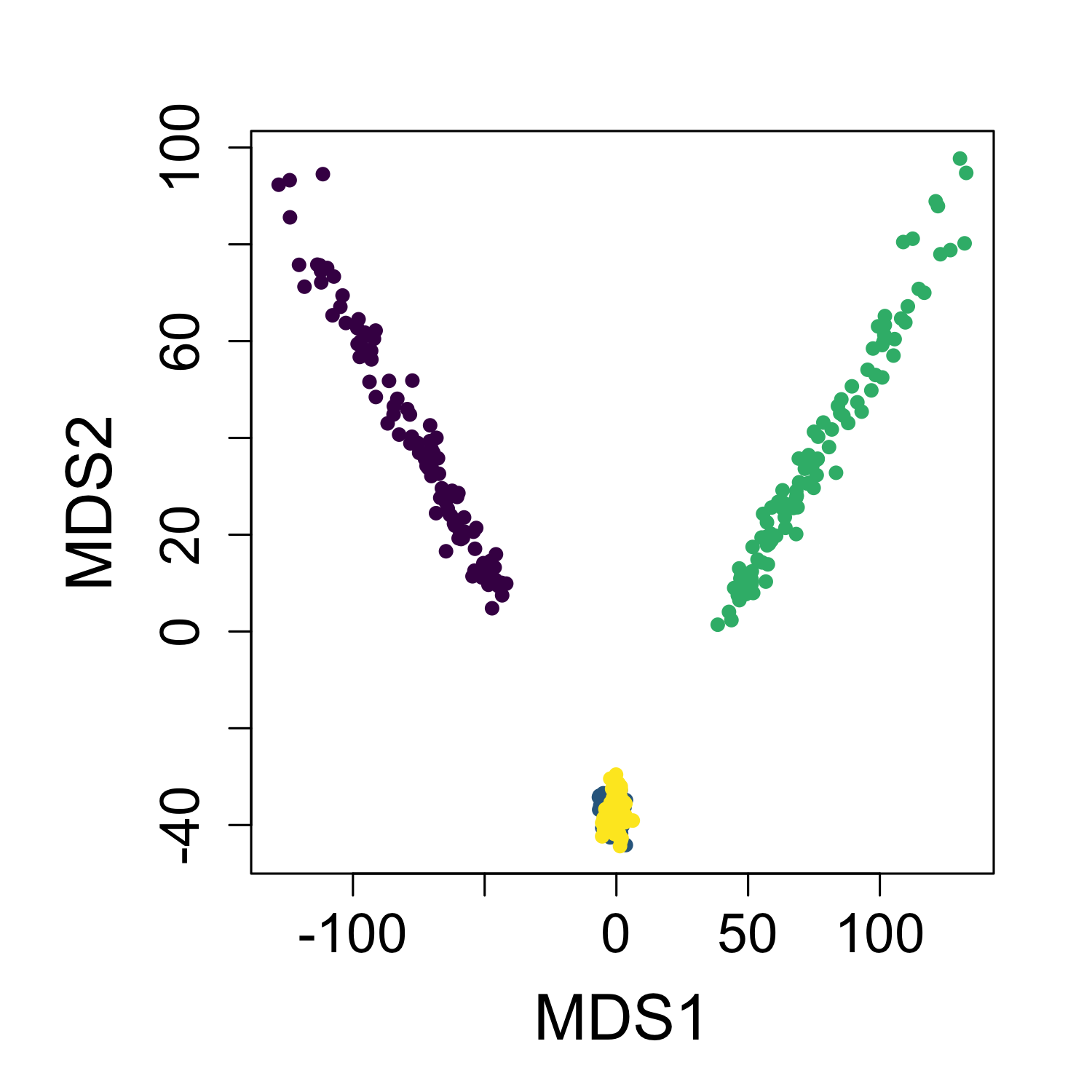}
	\end{minipage}
	\caption{Visualization of 400 randomly generated vMF distributions (left) and two-dimensional embeddings obtained via multidimensional scaling using the proposed $\WL$ distance (middle) and the standard $L_2$ distance (right). Colors indicate distribution types: \textsf{north-high} (\textcolor[HTML]{440154}{\textbf{purple}}),
		\textsf{north-low} (\textcolor[HTML]{31688E}{\textbf{blue}}),
		\textsf{south-high} (\textcolor[HTML]{35B779}{\textbf{green}}), and 
		\textsf{south-low} (\textcolor[HTML]{FDE725}{\textbf{yellow}}).}
	\label{fig:sim1_mds}
\end{figure}

The visualization of the 400 randomly generated vMF distributions and their two-dimensional embeddings is presented in Figure~\ref{fig:sim1_mds}. It is expected that a meaningful dissimilarity metric should ideally separate distributions that differ in either location or concentration. This behavior is clearly observed under the $\WL$ distance that all four distribution types form well-separated clusters, reflecting their underlying parameter differences. On the other hand, the embedding derived from the standard $L_2$ distance fails to distinguish between the \textsf{north-low} and \textsf{south-low} groups, which are merged into a single cluster. This failure can be attributed to the low concentration parameter $\kappa$ in these groups, which leads to broad dispersion around the mean direction. As $\kappa \to 0$, the vMF distribution approaches the uniform distribution on the sphere, and the effect of the mean direction becomes negligible. Consequently, the $L_2$ distance, which emphasizes pointwise differences in density, becomes less sensitive to directional variation in low-concentration regimes. By contrast, the $\WL$ distance explicitly incorporates both the angular separation of mean directions and the difference in concentration. This dual sensitivity enables it to successfully differentiate distributions across both directional and dispersion axes, thereby capturing the underlying geometric and probabilistic structure of the vMF family more effectively.

Our second simulated example illustrates the effectiveness of the proposed mixture reduction methods. We consider a 4-component vMF mixture model on $\bbS^1$ with equal weights, defined as 
\begin{equation*}
	\bbP_{\text{mix}} = \frac{1}{4} \sum_{i=1}^4 vMF(\bfmu_i, \kappa),
\end{equation*}
where the concentration parameter is fixed at $\kappa = 10$, and the mean directions $\bfmu_i$ are placed at the principal axes of the circle: $(1,0)$, $(0,1)$, $(-1,0)$, and $(0,-1)$. A random sample of size 400 was drawn from this model, with approximately 100 observations per component. Figure~\ref{fig:sim2_reduce_mix} (left) shows both the mixture density and the resulting sample, highlighting the four distinct, minimally overlapping modes.

\begin{figure}[ht]
	\centering
	\includegraphics[width=.9\linewidth]{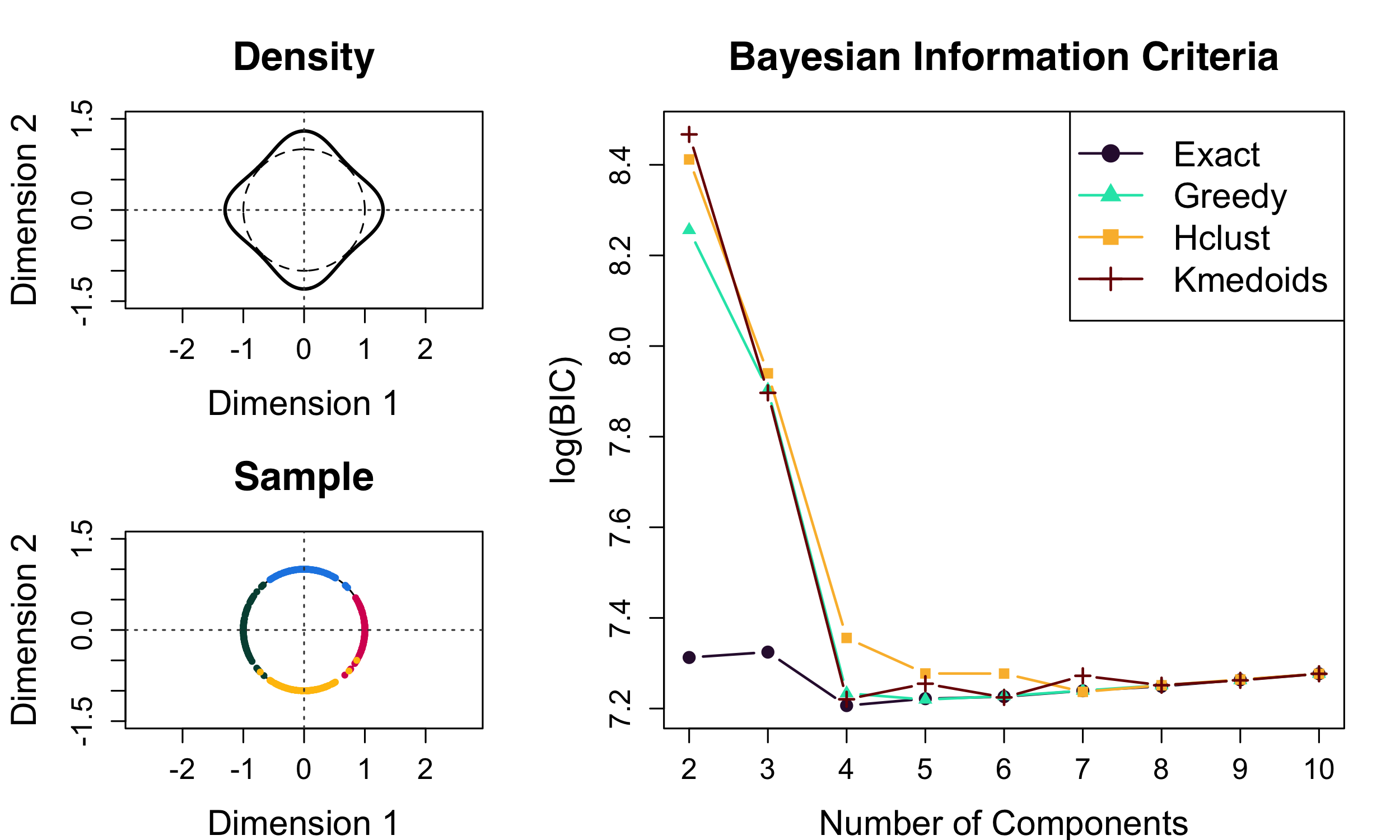}
	\caption{Simulated example of mixture model reduction. Left-top: density of the ground-truth 4-component vMF mixture model. Left-bottom: 400 randomly generated samples, color-coded by component membership. Right: log-transformed BIC values for independently fitted mixtures and reduced models using greedy and partitional methods across varying numbers of components.}
	\label{fig:sim2_reduce_mix}
\end{figure}

To evaluate the performance of our vMF mixture reduction algorithms, we first fitted finite mixture models for varying numbers of components $K \in {2, \ldots, 10}$ independently. Each fit was repeated 10 times with different initializations, and the model with the highest likelihood was retained. The best-fitting 10-component model served as the initialization for both the greedy and partitional reduction procedures, which were applied iteratively to compress the model down to 2 components. 

As an evaluation metric, we computed the Bayesian Information Criterion (BIC) \citep{schwarz_1978_EstimatingDimensionModel}, which for a vMF mixture with $k$ components in dimension $d$ is given by 
\begin{equation*}
	\textrm{BIC} = -2 \log (\hat{L}) + (k(d+1)-1)\log n,
\end{equation*}
where $\hat{L}$ denotes the maximized likelihood and $n$ is the sample size. The BIC values are summarized in Figure~\ref{fig:sim2_reduce_mix} (right). All methods, except for the partitional method with hierarchical clustering, achieved the lowest BIC at $K = 4$, which corresponds to the true number of components. Notably, as the number of components decreases below 4, the BIC values of the reduced models increase sharply, often at a much steeper rate than those of independently fitted models. This divergence highlights the utility of the proposed reduction techniques that they preserve model structure effectively until reaching the intrinsic complexity of the data, at which point any further compression results in clear degradation.

This behavior offers an interesting perspective on model selection. The rapid increase in BIC below the true model size may serve as a form of negative feedback that discourages under-specification. For example, the greedy method correctly recovers all four true component locations at $K=4$, but further merging leads to configurations with imbalanced components - one large and two small - which distort the data-generating structure. In contrast, fitting a 3-component mixture from scratch often yields an artificial partition of the circle into three evenly spaced segments to distribute mass uniformly. This contrast illustrates how our reduction framework is not only computationally efficient but also intrinsically penalizes structural misspecification.

\subsection{Real Example 1 : Abstract Classification}

Our first real-data experiment focuses on text classification using pretrained sentence embedding models \citep{patil_2023_SurveyTextRepresentation}. We employ the medical abstract dataset from \citet{schopf_2022_EvaluatingUnsupervisedText}, originally designed for zero-shot and similarity-based classification tasks. The corpus consists of 14,438 abstracts, each labeled into one of five disease categories: neoplasms (3163), digestive system diseases (1494), nervous system diseases (1925), cardiovascular diseases (3051), and general pathological conditions (4805). Although the original dataset includes a predefined train/test split, we combine the full corpus to facilitate custom sampling for our analysis.

A standard pipeline in text-based statistical learning is to embed documents into vector space representations prior to downstream modeling. While this approach is often applied at the phrase or sentence level, modern embedding models are capable of handling paragraph-level inputs as well \citep{le_2014_DistributedRepresentationsSentences}. However, given that abstracts typically comprise multiple thematically distinct sentences covering background, objectives, and contributions, it is more natural to model an abstract as a bag of semantically diverse yet coherent sentences \citep{harris_1954_DistributionalStructure}.

This motivates our novel representation pipeline: each abstract is split into sentences, each sentence is embedded using a pre-trained model, and the resulting vectors are $\ell_2$-normalized and modeled as a sample from a vMF  distribution. This provides a compact, distributional representation of the abstract via a unit-norm mean direction and scalar concentration parameter. To our knowledge, this approach has not been explicitly explored in the literature. We use the \textsf{mpnet-base-v2} model, a pretrained sentence embedding model  based on the \textsf{mpnet} architecture \citep{song_2020_MPNetMaskedPermuted}, which is available from the \textsf{Sentence-Transformers} library \citep{reimers_2019_SentenceBERTSentenceEmbeddings}. It encodes sentences into 768-dimensional vectors. On average, each abstract contains 8.59 sentences with standard deviation of 3.26, ranging from 2 to 27. Although the ambient dimension exceeds the typical number of sentences per abstract, the vMF model’s use of a single scalar concentration parameter makes it robust in this low-sample regime.

\begin{figure}[ht]
	\centering
	\includegraphics[width=.95\linewidth]{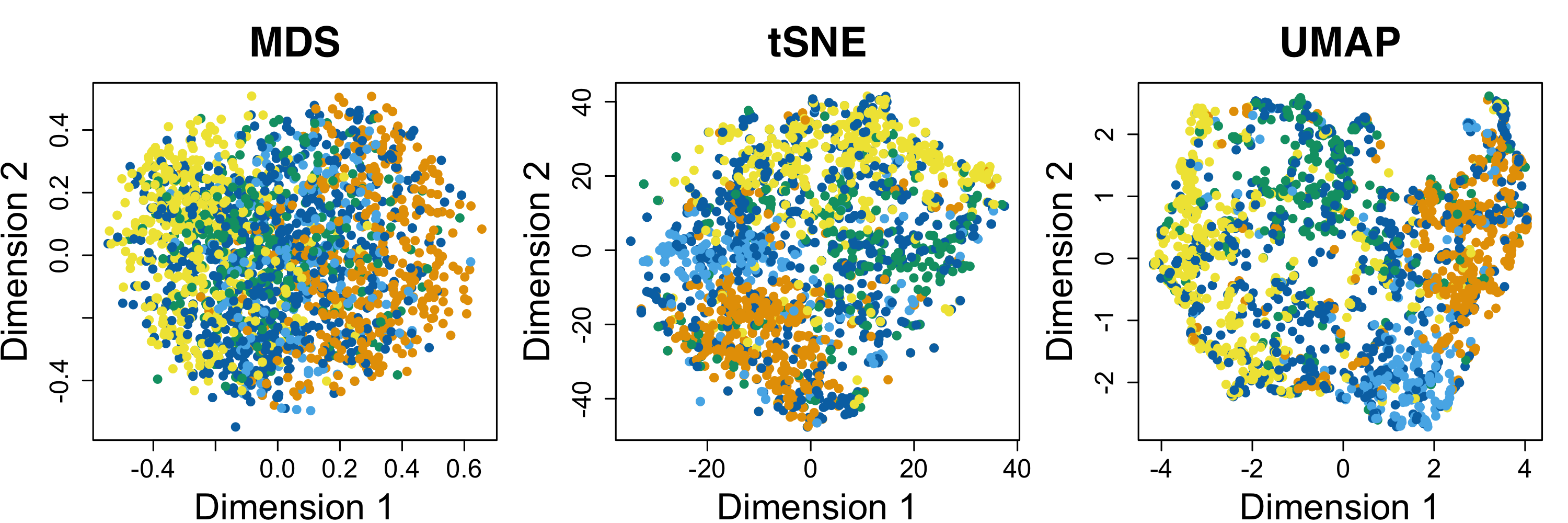}
	\caption{Two-dimensional embeddings of 1000 abstracts represented by vMF distributions over sentence embeddings. Embedding techniques include multidimensional scaling (MDS), t-stochastic neighbor embedding (t-SNE), and uniform manifold approximation and projection (UMAP). Colors correspond to the five ground-truth disease categories.}
	\label{fig:real1_med_embedding}
\end{figure}

As an exploratory analysis, we randomly selected 1,000 abstracts and computed pairwise dissimilarities between their vMF representations using the proposed $\WL$ distance. The resulting distance matrix was projected to $\mathbb{R}^2$ using three nonlinear embedding algorithms: MDS, t-SNE \citep{vandermaaten_2008_VisualizingDataUsing}, and UMAP \citep{mcinnes_2020_UMAPUniformManifold}. Figure~\ref{fig:real1_med_embedding} shows the resulting embeddings, which  reveal only mild separation across categories, with substantial overlap between clusters. This entanglement reflects genuine semantic and clinical proximity among disease classes. Many biomedical abstracts reference multiple physiological systems or describe comorbid conditions, contributing to blurred boundaries in the latent space. Thus, the lack of clear separation is not necessarily a failure of the embedding model or metric, but rather a feature of the domain itself.

We then conducted a classification task using the same vMF representation pipeline. From the full dataset, we randomly sampled 2000 abstracts for training and another 2000 for testing. Each abstract was embedded as a vMF distribution over its sentence vectors, and classification was performed using $k$-nearest neighbors (KNN) for $k \in {1,\ldots,10}$. Distances between vMF distributions were computed using both the $\WL$ and $L_2$ metrics. The latter was estimated via Monte Carlo approximation with adaptive sample size (terminated when incremental change fell below $10^{-6}$). As a comparison, we also evaluated a vector-based representation via mean-pooling sentence embeddings followed by $\ell_2$ normalization, with cosine similarity used for KNN.

To benchmark performance, we implemented three baseline feature extraction methods: bag-of-words (BoW), term frequency-inverse document frequency (TF-IDF), and Word2Vec \citep{mikolov_2013_EfficientEstimationWord}. Preprocessing steps, including lowercasing, removal of punctuation, and stemming, were applied to the texts, followed by tokenization  using the \textsf{text2vec} package \citep{selivanov_2023_Text2vecModernText} in \textsc{R}, and embedded accordingly. BoW and TF-IDF produced sparse frequency-based representations; for Word2Vec, we trained a skip-gram model via the \textsf{word2vec} package \citep{wijffels_2023_Word2vecDistributedRepresentations} and averaged word embeddings for each document. All feature sets were used to train XGBoost classifiers \citep{chen_2016_XGBoostScalableTree} with consistent hyperparameters across methods, including maximum depth 6, learning rate $\eta=0.1$, and subsampling rates of 0.8 for rows and columns. Model selection was done via 5-fold cross-validation with early stopping based on multiclass log loss. The learned vectorization objects and XGBoost models were applied to the held-out test data for final evaluation, and predicted class indices were mapped back to the original label levels for interpretability. All experiments were repeated five times; mean performance is summarized in Figure~\ref{fig:real1_med_report} using four evaluation metrics: precision, recall, F1 score, and specificity, reported as weighted averages to account for class imbalance.

\begin{figure}[ht]
	\centering
	\includegraphics[width=.95\linewidth]{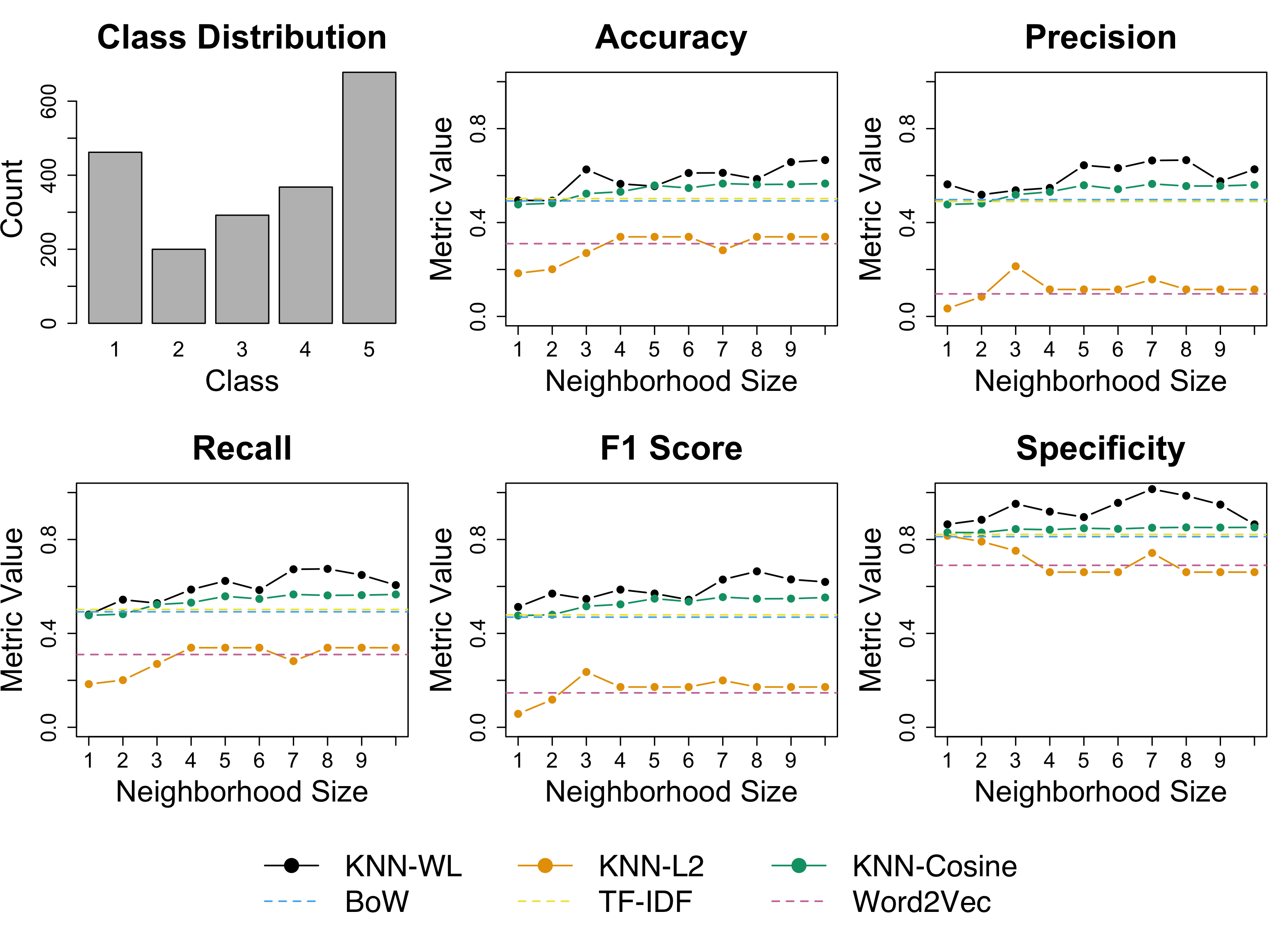}
	\caption{Classification performance comparison for multiple methods. Metrics include weighted averages of precision, recall, F1 score, and specificity. Methods include KNN classifiers using $\WL$, $L_2$, and cosine distances, and XGBoost classifiers using BoW, TF-IDF, and Word2Vec features.}
	\label{fig:real1_med_report}
\end{figure}

The results show that KNN with $\WL$ distance consistently outperforms all other methods, including KNN with $L_2$ distance and cosine similarity. The $L_2$-based vMF classification performed notably worse, reflecting its inadequacy in capturing concentration-based differences between distributions. Even the cosine similarity model, which performed reasonably well, was surpassed in most cases by the $\WL$-based approach. Among the baselines, Word2Vec performed poorest, with accuracy comparable to $L_2$-based KNN. Both BoW and TF-IDF achieved similar levels of performance—slightly below cosine similarity—despite using XGBoost classifiers and cross-validation. Notably, this performance came at greater computational cost, highlighting the efficiency of the $\WL$-based KNN pipeline. Finally, we note the near-identical performance of BoW and TF-IDF. This is consistent with the fact that they share the same tokenization pipeline, and XGBoost is known to be invariant to monotonic transformations like IDF scaling. In low-data settings with high-dimensional sparse features, IDF reweighting has limited benefit. In conclusion, the vMF-based representation combined with the $\WL$ distance offers a robust and scalable alternative to standard text classification pipelines. Its structural sensitivity to both directionality and concentration makes it especially effective in modeling abstract-level semantic variation.

\subsection{Real Example 2 : Image Clustering}

Our second real-data experiment involves model-based clustering of image data via vMF distributions on the unit hypersphere. Specifically, we use deep visual representations obtained from a pretrained Vision Transformer (ViT) model, with the CIFAR-10 dataset serving as the primary benchmark. CIFAR-10 is a widely used dataset in computer vision, consisting of 60,000 low-resolution color images evenly distributed across ten classes: airplane, automobile, bird, cat, deer, dog, frog, horse, ship, and truck \citep{krizhevsky_2009_LearningMultipleLayers}. Each image is of size 32×32 pixels, making the dataset particularly suitable for evaluating representation learning and clustering under constrained visual resolution.

To ensure compatibility with the ViT architecture, images were resized to 224×224 pixels. We used the \textsf{google/vit-base-patch16-224} model \citep{dosovitskiy_2021_ImageWorth16x16}, publicly available via the Hugging Face \textsf{Transformers} library . This model processes each image by dividing it into a grid of 16×16 non-overlapping patches with 196 total, which are linearly projected into a high-dimensional embedding space. A learnable \textsf{[CLS]} token is prepended to aggregate global information, and positional embeddings are added to encode spatial relationships. The resulting sequence is passed through multiple Transformer encoder blocks. The output corresponding to the \textsf{[CLS]} token yields a 768-dimensional vector that serves as a compact and semantically rich embedding of the entire image.

\begin{figure}[ht]
	\centering
	\includegraphics[width=.925\linewidth]{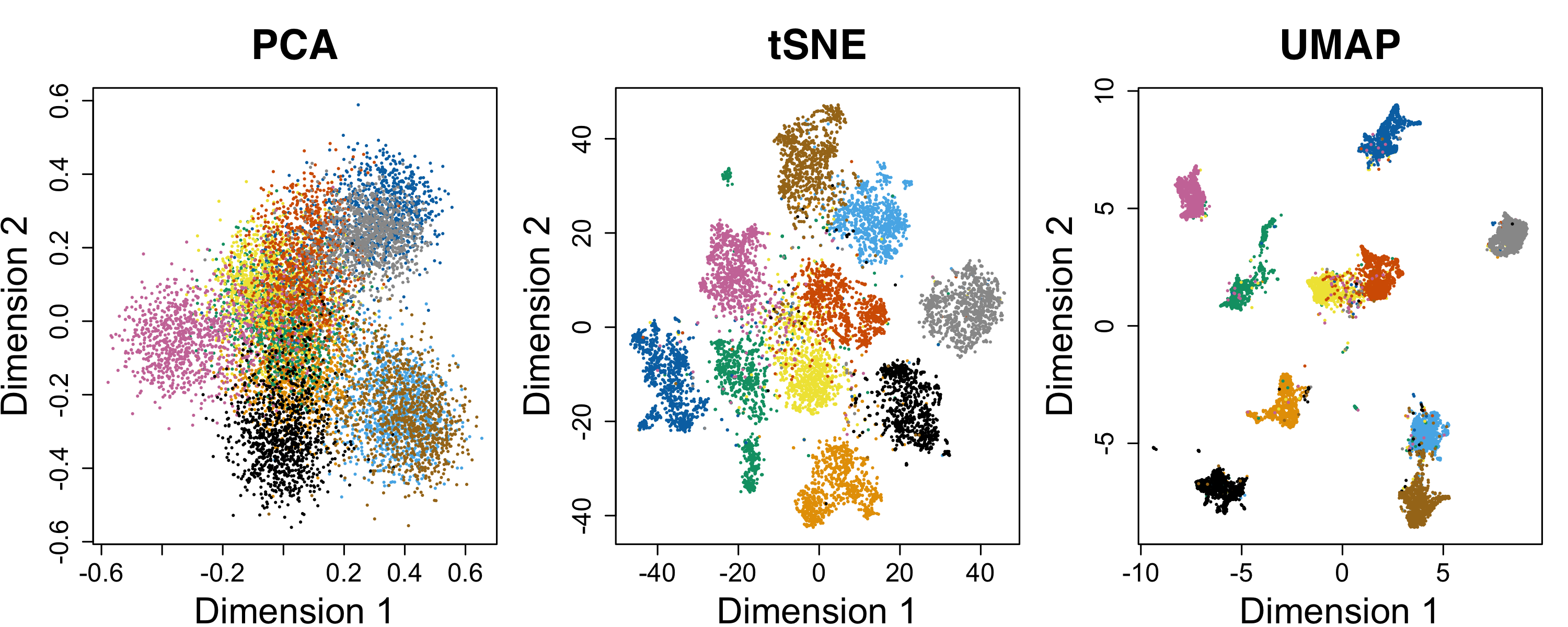}
	\caption{Two-dimensional embeddings of CIFAR-10 image embeddings produced by ViT and normalized to lie on the unit hypersphere. Projections were obtained via principal component analysis (PCA), t-stochastic neighbor embedding (tSNE), and uniform manifold approximation and projection (UMAP). Colors indicate true class labels.}
	\label{fig:real2_cifar_embedding}
\end{figure}

We randomly selected 10000 images from CIFAR-10 and applied the above embedding pipeline, including $\ell_2$ normalization, to produce a dataset in $\mathbb{R}^{10000 \times 768}$. This preprocessing ensured that the data lay on the unit hypersphere, enabling direct modeling with vMF distributions. As shown in Figure~\ref{fig:real2_cifar_embedding}, class separation is visually apparent in the projected embeddings obtained via PCA, t-SNE, and UMAP. While PCA shows some overlap, t-SNE and UMAP offer more discernible cluster boundaries, supporting the use of geometric methods in subsequent modeling.

We then fit a vMF mixture model with $K=20$ components, twice the true number of CIFAR-10 classes. This model was fit independently five times to mitigate initialization sensitivity, and the run with the lowest Bayesian Information Criterion (BIC) was retained as the reference mixture. From this reference, we applied the three proposed model reduction techniques - greedy merging, hierarchical clustering, and $k$-medoids - to reduce the number of components from $K=19$ to $K=2$. To evaluate the effectiveness of these reduction techniques, we also fit vMF mixture models independently for each $K \in {2, \dots, 20}$, selecting the best run by BIC in each case. For benchmarking, we applied standard clustering methods as well, including $k$-means \citep{macqueen_1967_MethodsClassificationAnalysis} and spherical $k$-means \citep{dhillon_2001_ConceptDecompositionsLarge}, to assess cluster recovery independent of probabilistic modeling.

\begin{figure}[ht]
	\centering
	\includegraphics[width=.925\linewidth]{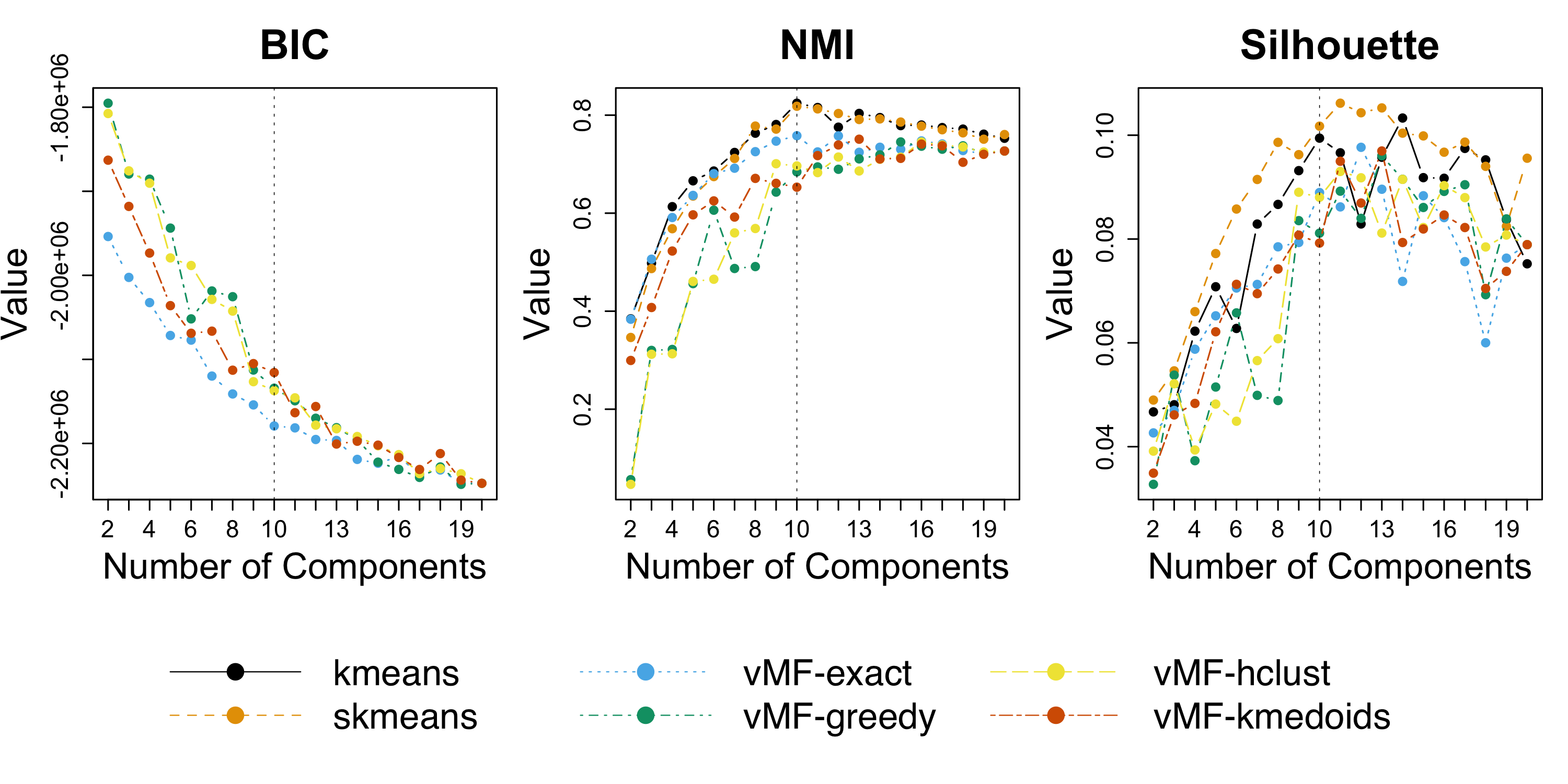}
	\caption{Comparison of model fit and clustering quality for various methods. Metrics include Bayesian Information Criterion (BIC), normalized mutual information (NMI), and Silhouette score. Vertical dashed line indicates the true number of CIFAR-10 classes.}
	\label{fig:real2_cifar_report}
\end{figure}

Figure~\ref{fig:real2_cifar_report} summarizes the performance of all methods. For model fit, the vMF mixtures trained independently achieve the lowest BIC values across all $K$, as expected. However, the gap between exact and reduced models narrows for $K \geq 10$, indicating diminishing returns from exact inference beyond the true cluster size. Below $K=10$, all reduction methods exhibit a sharp rise in BIC, echoing patterns observed in our simulated examples and indicating a loss of model fidelity when overcompressing the mixture.

In terms of clustering performance, both $k$-means and spherical $k$-means outperform all vMF-based methods in normalized mutual information (NMI), particularly near the true number of clusters ($K = 10$). This result is not unexpected, as the isotropic dispersion assumption inherent in vMF models can be overly restrictive in high-dimensional settings and may fail to capture the anisotropic structure present in deep embeddings, as illustrated in Figure~\ref{fig:real2_cifar_embedding}. Nevertheless, the reduced vMF models perform comparably to the exact vMF mixtures for $K \geq 10$, suggesting that our reduction strategies preserve much of the representational capacity of the original mixture model.

Silhouette scores provide a complementary internal evaluation that does not rely on ground truth class labels. Spherical $k$-means achieves the highest Silhouette scores, indicating compact and well-separated clusters. Among vMF-based methods, no consistent dominance is observed; in fact, the exact vMF mixture model often yields the lowest Silhouette scores across various values of $K$. Notably, all reported Silhouette scores for vMF-based methods lie within a narrow range of $[0.02, 0.11]$, far from the ideal value of 1.0. One possible explanation for this phenomenon is the limited expressiveness of the pretrained ViT model, which, while capable of recovering class structure to some extent as evidenced by NMI, may lack sufficient task-specific refinement to induce strongly separated decision boundaries in the embedding space. This limitation is consistent with the known behavior of ViT models, which often benefit from fine-tuning on large task-relevant corpora to improve semantic specificity and downstream performance.

In summary, while $k$-means and spherical $k$-means offer the strongest clustering performance, our vMF reduction framework provides a principled and scalable alternative for modeling spherical data distributions. Although the isotropic assumption imposes limitations, the geometry-aware reduction methods retain fidelity to the original model and offer interpretable compression paths. This underscores the utility of the $\WL$-based geometry in clustering and mixture model simplification.

\section{Conclusion}

In this paper, we introduced a novel geometry-inspired dissimilarity metric for vMF distributions, grounded in high-concentration approximations and optimal transport theory. By interpreting the vMF distribution through a Gaussian lens on the tangent space of the unit hypersphere, we derived a closed-form expression for a Wasserstein-like distance, denoted $\WL$, which cleanly separates into angular and dispersion-based components. This formulation not only respects the intrinsic curvature of spherical domains but also yields a tractable and interpretable tool for comparing directional distributions.

The proposed distance equips the space of non-degenerate vMF laws with a meaningful geometric structure, enabling principled extensions of classical statistical operations such as barycenter computation. We leveraged this structure to develop two model reduction strategies, greedy and partitional, that compress large mixtures of vMF components while thriving to preserve underlying structure. Our experimental evaluation, encompassing both synthetic data and high-dimensional real-world embeddings, demonstrates that $\WL$ consistently improves on standard metrics like $L_2$ distance in classification, clustering, and model selection tasks. Notably, the $\WL$-based methods offer both theoretical soundness and computational efficiency in scenarios where existing approaches rely on approximation or struggle with scalability.

This work provides a foundation for several lines of future inquiry. One direction is a deeper theoretical investigation into the relationship between the proposed $\WL$ metric and the exact Wasserstein distance for vMF distributions. Another is the extension of our approach to broader families of directional distributions, especially those exhibiting anisotropy or more complex dependence structures. On the algorithmic side, the integration of $\WL$ into learning frameworks such as Riemannian variational inference, spherical generative models, or manifold-aware neural architectures could significantly enhance performance in geometric learning tasks. Ultimately, we believe the tools developed here contribute to a growing statistical and algorithmic framework for inference on curved spaces, with applications ranging from natural language and computer vision to biomedicine and beyond.

\bibliographystyle{dcu}
\bibliography{references}


\end{document}